\documentclass[letterpaper, 10 pt, conference]{IEEEconf}
\IEEEoverridecommandlockouts
\usepackage{graphicx} 

\usepackage{amsmath, amsfonts, amssymb}
\usepackage{url}
\usepackage{algorithm}
\usepackage{algpseudocode}
\usepackage{subcaption}
\usepackage{rotating}

\usepackage{mathtools} 

\renewcommand{\Re}{\mathbb{R}}
\newcommand{\Expectation}{\mathop{{}\mathbb{E}}}

\usepackage{todonotes}

\usepackage{amsthm}  

\newtheorem{lemma}{Lemma}
\newtheorem{theorem}{Theorem}
\newtheorem{proposition}{Proposition}
\newtheorem{remark}{Remark}
\newtheorem{corollary}{Corollary}

\makeatletter
\let\NAT@parse\undefined
\makeatother
\usepackage{hyperref}
\hypersetup{
   colorlinks=true,
    linkcolor= blue,
    allcolors=blue,
    citecolor = blue,
    filecolor=black,
    urlcolor=blue,
    }

\title{Dual Control for Interactive Autonomous Merging \\ with Model Predictive Diffusion}

\author{Jacob Knaup,
Jovin~D'sa,
Behdad~Chalaki,
Hossein~Nourkhiz Mahjoub,
Ehsan~Moradi-Pari,
Panagiotis Tsiotras
\thanks{\hangindent=0.5cm J. Knaup and P. Tsiotras are with the Institute for Robotics and Intelligent Machines,
Georgia Institute of Technology,  Atlanta, GA 30332--0250 USA. This work was conducted during J. Knaup's internship at Honda Research Institute. (email: jacobk@gatech.edu)}
\thanks{J. D'sa, B. Chalaki, H. N. Mahjoub, and E. Moradi-Pari are with Honda Research Institute, USA Inc. (email: jovin{\_}dsa@honda-ri.com)}
}

\begin{document}

\maketitle

\begin{abstract}
Interactive decision-making is essential in applications such as autonomous driving, where the agent must infer the behavior of nearby human drivers while planning in real-time. Traditional predict-then-act frameworks are often insufficient or inefficient because accurate inference of human behavior requires a continuous interaction rather than isolated prediction. 
To address this, we propose an active learning framework in which we rigorously derive predicted belief distributions.
Additionally, we introduce a novel model-based diffusion solver tailored for online receding horizon control problems, demonstrated through a complex, non-convex highway merging scenario. 
Our approach extends previous high-fidelity dual control simulations to hardware experiments, which may be viewed at \url{https://youtu.be/Q_JdZuopGL4}, and verifies behavior inference in human-driven traffic scenarios, moving beyond idealized models. 
The results show improvements in adaptive planning under uncertainty, advancing the field of interactive decision-making for real-world applications. 
    
\end{abstract}

\section{Introduction}

\subsection{Motivation}
Robots, such as autonomous vehicles (AV), need to safely and efficiently interact with humans and make decisions based on human intentions and behaviors. In highly dynamic scenarios, these interactions become particularly complex, as AVs must anticipate and respond to the unpredictable actions of human drivers whose intentions are often unobservable or observed too late. One such challenging scenario is highway on-ramp merging, where the safe and efficient merging of an AV depends on correctly understanding and predicting human drivers’ intention in real time. Relying on strong assumptions about human behavior models, as traditional Model Predictive Control (MPC) approaches often do, may result in suboptimal and potentially unsafe merging outcomes. 

To address this challenge, we propose a dual control framework that aims to accomplish two objectives: performing an efficient merging maneuver while actively reducing uncertainty about other human drivers' intentions. Dual control enables the system to balance exploration—probing the behavior of other vehicles to better understand their intent, which may not be readily observable, and exploitation—leveraging known information to take optimal control actions. By combining dual control with online Bayesian inference, the proposed approach dynamically updates the belief distribution over human driver behaviors, enabling the AV to make more informed and adaptive control decisions. This strategy not only improves merging efficiency but also enhances safety by continuously refining the AV's understanding of the traffic environment throughout the merging process.

\subsection{Related Work}
MPC has been extensively utilized in the planning and control of robots and autonomous vehicles due to its ability to handle constraints and optimize control actions over a finite horizon \cite{yu2021model}. Recently, several learning-enhanced MPC formulations have been proposed, as highlighted in the survey \cite{aradi2020survey}. Some approaches have integrated learning-based prediction models into the MPC framework \cite{le2024multi}, while others have proposed combining Reinforcement Learning with MPC to better handle the complex, and hard-to-model interactions between autonomous vehicles and humans \cite{albarella2023hybrid,kimura2022decision}. Many of these approaches rely on a ``predict-then-act" framework, where the system predicts future states based on current observations and then acts accordingly. However,  this sequential framework can be insufficient in highly dynamic environments, where continuous interaction and adaptation to human behavior are required. To address this limitation, interaction-aware planners have been developed, such as those by authors in \cite{le2024multi,gupta2023interaction}. Similarly, planners capable of probing human agents to infer their intentions and acting interactively have demonstrated greater effectiveness in highly uncertain environments, as shown in \cite{wang2023active}. 

Several previous works have explored the use of dual control for interactive autonomous driving.
In \cite{nair2022stochastic}, the authors developed a dual MPC framework that actively learns the behavior of other drivers, characterized by a set of basis functions with unknown weights. 
These weights are learned online using a Kalman filter and applied a linear representation for belief propagation.
In the same work, the authors reformulated the dual control problem as a convex second-order cone program, enabling efficient computation of solutions.
However, the linearity assumptions required for belief propagation and convex problem formulation may fail to capture the complex dynamics necessary for accurately modeling the problem.

In \cite{hu2022active} the authors formulated a dual control problem for interactive driving with uncertainty about human driver behavior.
Similar to \cite{nair2022stochastic}, they parameterized the unknown human behavior as a linear combination of known basis functions with unknown weights.
The authors in \cite{hu2022active} approximated the belief propagation using a Gaussian parameterization and solved the resulting dual control program using nonlinear programming (NLP). In \cite{hu2024active}, the same authors extended this work by incorporating a safety filter into the dual control policy to mitigate accidents caused by improbable but high-risk events.

Generative models, particularly those utilizing diffusion processes, have demonstrated to be effective in planning and control applications \cite{janner2022planning, chi2023diffusion, sun2024conformal, psenka2023learning}. 
These models are capable of producing new samples from complex distributions, making them well-suited for addressing non-convex and multi-modal challenges \cite{ho2020denoising,song2019generative}. 
In this work, we introduce a new variant of a model-based diffusion solver, specifically designed for receding horizon optimization, which effectively manages the complexities of autonomous highway merging scenarios. 
Recent studies, such as Pan et al. \cite{pan2024model} have highlighted the effectiveness of model-based diffusion in solving trajectory optimization problems. 
For a comprehensive review of diffusion models and their applications, we refer the reader to \cite{yang2023diffusion}, while the foundational derivations for score-based generative modeling through SDEs are detailed in \cite{song2020score}, providing the basis for the proposed approach.

Highway on-ramp merging is widely recognized as a particularly challenging task for both human drivers and AVs due to the need to negotiate with other drivers under tight time and lane-ending constraints  \cite{rios2016survey,fernandez2021highway}. 
In our previous work \cite{knaup2024active}, we formulated a dual control framework that used model predictive path integral control (MPPI) to solve this problem. 
While our previous framework showed good performance, it relied on the parametric MR-IDM model \cite{holley2023mr}, which performed well in idealized scenarios but had limitations when dealing with the variability and unpredictability of real-world human driver behaviors.

\subsection{Contributions of this Work}
The main contributions of this paper are as follows:
\begin{enumerate}
\item We introduce a dual control framework that generates optimal control actions to achieve the desired goal while also actively reducing uncertainty over the belief distribution of human intents using online Bayesian inference. 
This approach enables intelligent navigation of the trade-off between exploration and exploitation, making it well-suited for complex interaction challenges like autonomous highway merging. 
We expand upon our previous work \cite{knaup2024active} by adding additional theoretical results in support of our Bayesian estimation and belief prediction methodology.

\item We propose a novel variant of a model-based diffusion solver, specifically designed for receding horizon optimization which is highly effective in handling the non-convex multi-modal interaction aspects of the problem. 
This contribution extends the work of \cite{pan2024model} by introducing a dynamic prior distribution which is more suitable for receding horizon control than the authors' original methodology which targeted static, offline optimization.

\item We validate our framework through real-world hardware experiments on F1-Tenth cars, demonstrating the effectiveness of our framework in challenging, real-time traffic merging scenarios. 
This is the first application of model-based diffusion to autonomous driving, an application for which its global solution capabilities are highly suitable due to the non-convex nature of the solution space owing to the presence of multiple dynamic obstacles (e.g., other vehicles).

\item We demonstrate that our framework can effectively adapt its control actions for better interaction with drivers by testing with human-controlled vehicles, addressing the limitation of relying solely on predefined driver models which may not perform well in out-of-distribution scenarios.

\end{enumerate}

\subsection{Outline}

The rest of this paper is organized as follows.
In Section~\ref{sec:problem_formulation}, we introduce the stochastic optimal control problem which this paper addresses, and we introduce the framework for online belief estimation of uncertain parameters.
Section~\ref{sec:proposed_dual_control} introduces the proposed dual control problem formulation,
Section~\ref{sec:proposed_model_based_diffusion} presents the proposed model predictive diffusion algorithm, 
and Section~\ref{sec:proposed_sampling_approximation} summarizes the sampling-based approximations which are used to efficiently solve the dual control problem using model predictive diffusion.
Section~\ref{sec:interactive_autonomous_driving} introduces the interactive autonomous driving application on which we evaluate the proposed approach and presents the experimental results.
Finally, we conclude in Section~\ref{sec:conclusion}.

\begin{figure*}[th]
    \centering
    \begin{subfigure}{0.99\columnwidth}
        \includegraphics[width=\textwidth]{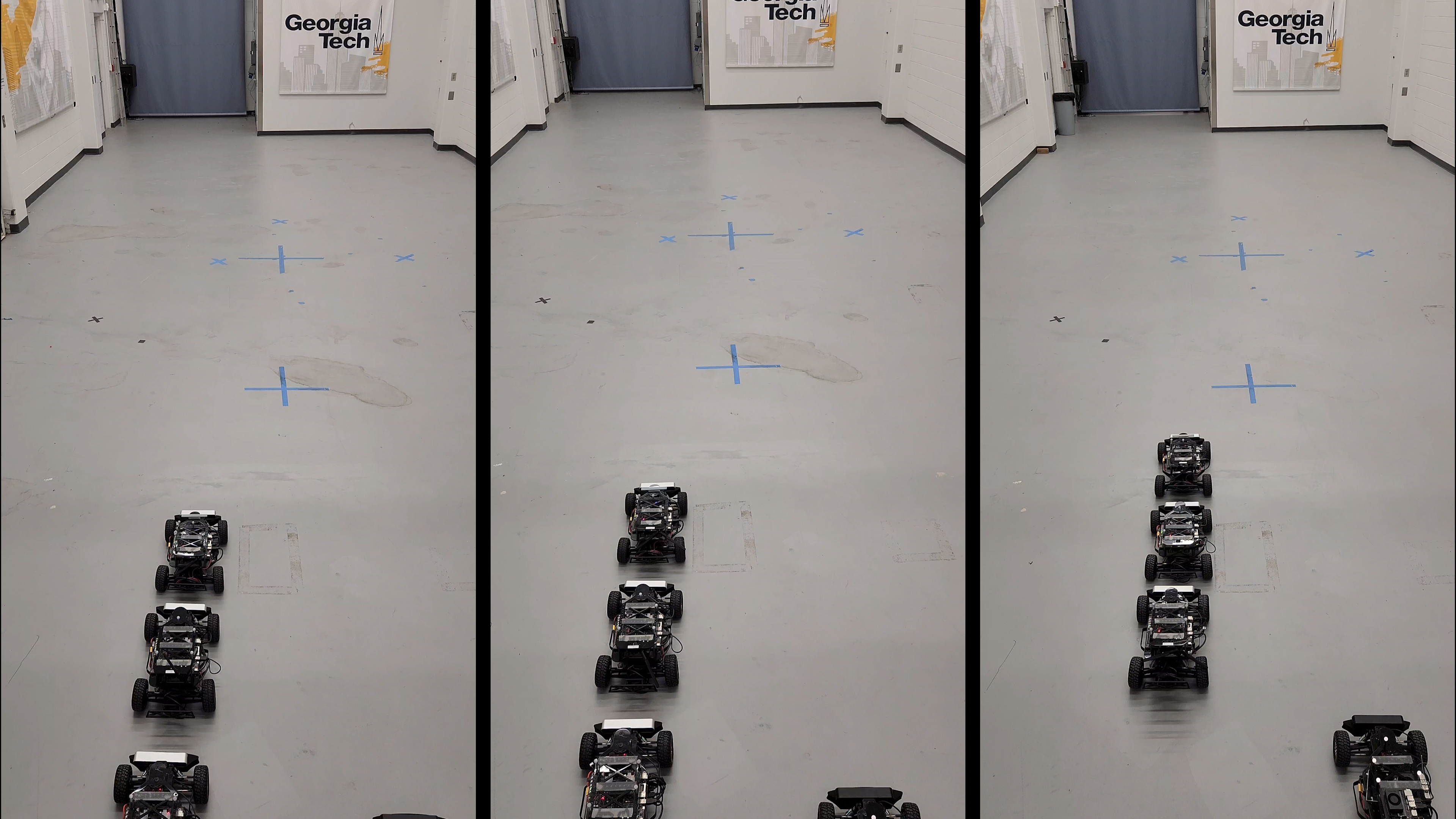}
    \end{subfigure}
    \begin{subfigure}{0.99\columnwidth}
        \includegraphics[width=\textwidth]{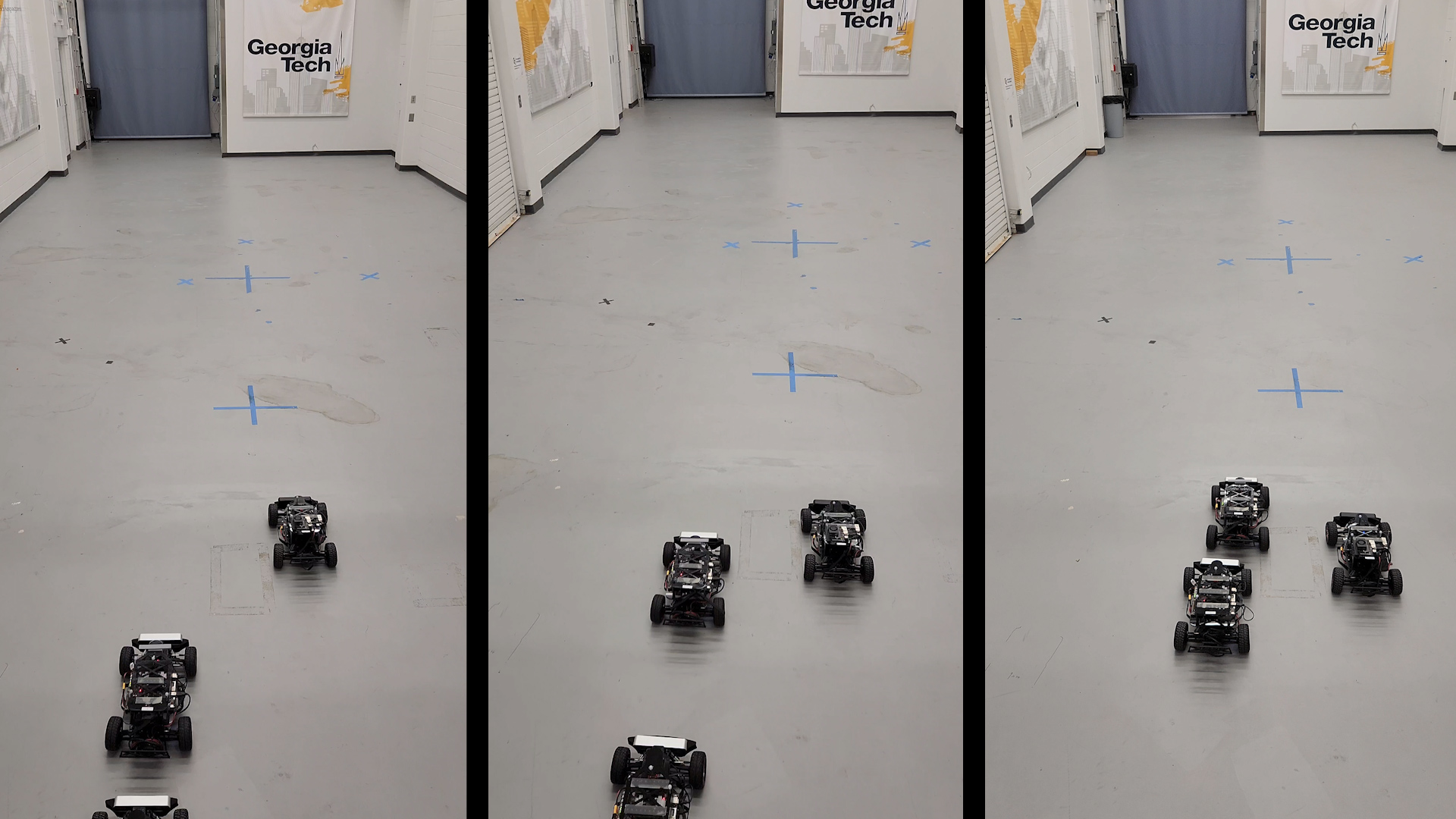}
    \end{subfigure}
    \begin{subfigure}{0.99\columnwidth}
        \includegraphics[width=\textwidth]{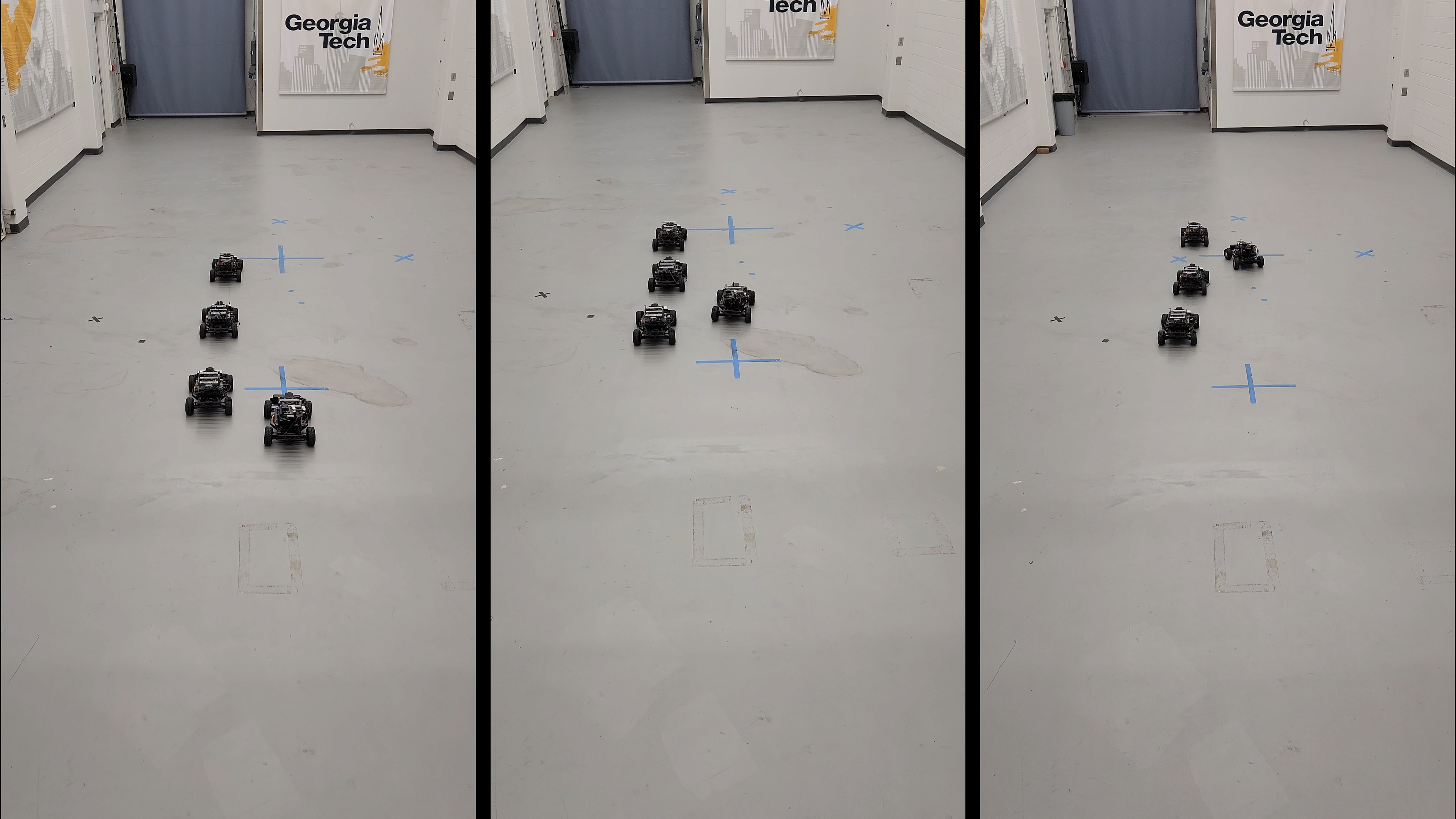}
    \end{subfigure}
    \begin{subfigure}{0.99\columnwidth}
        \includegraphics[width=\textwidth]{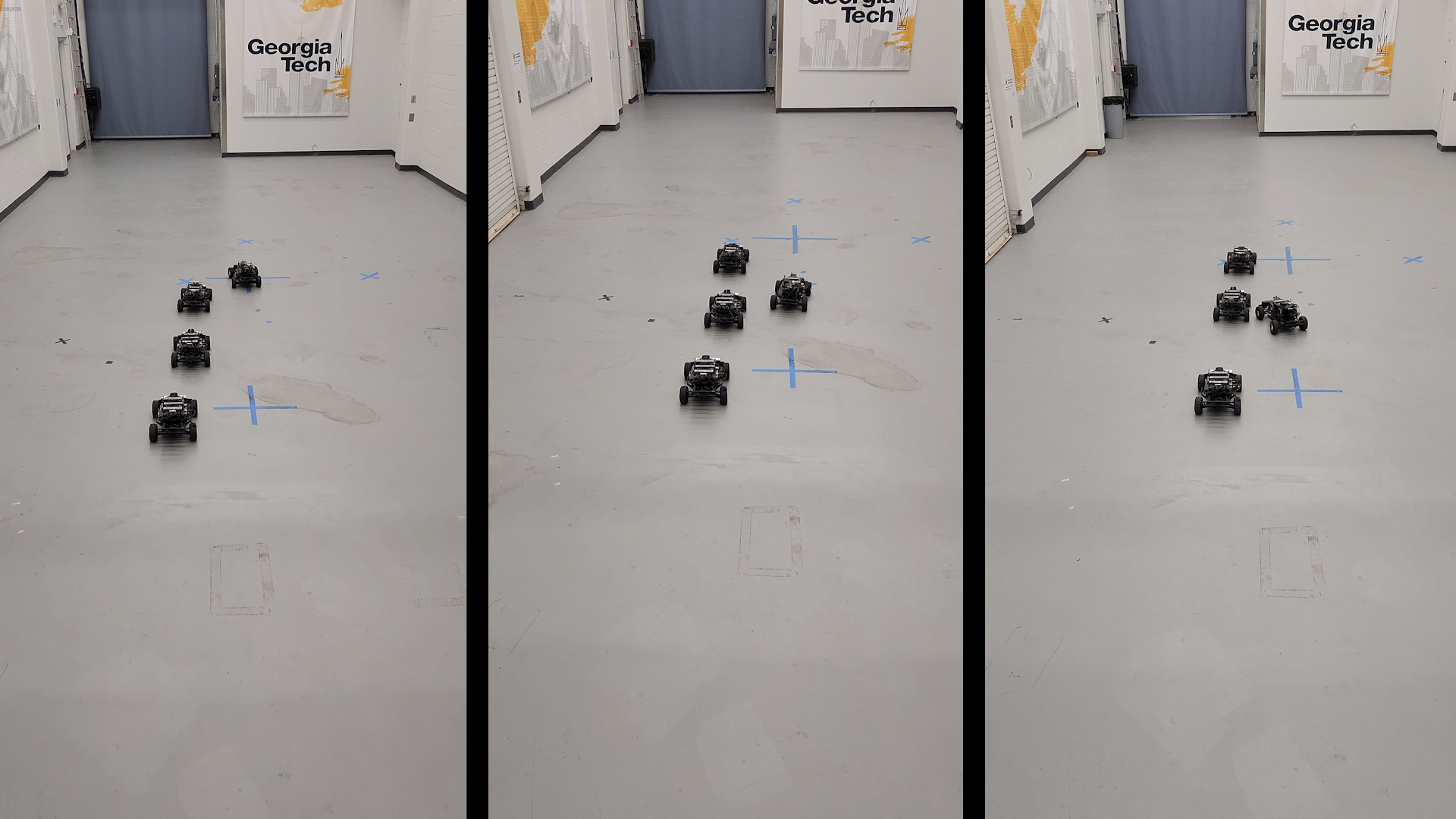}
    \end{subfigure}
    \begin{subfigure}{0.99\columnwidth}
        \includegraphics[width=\textwidth]{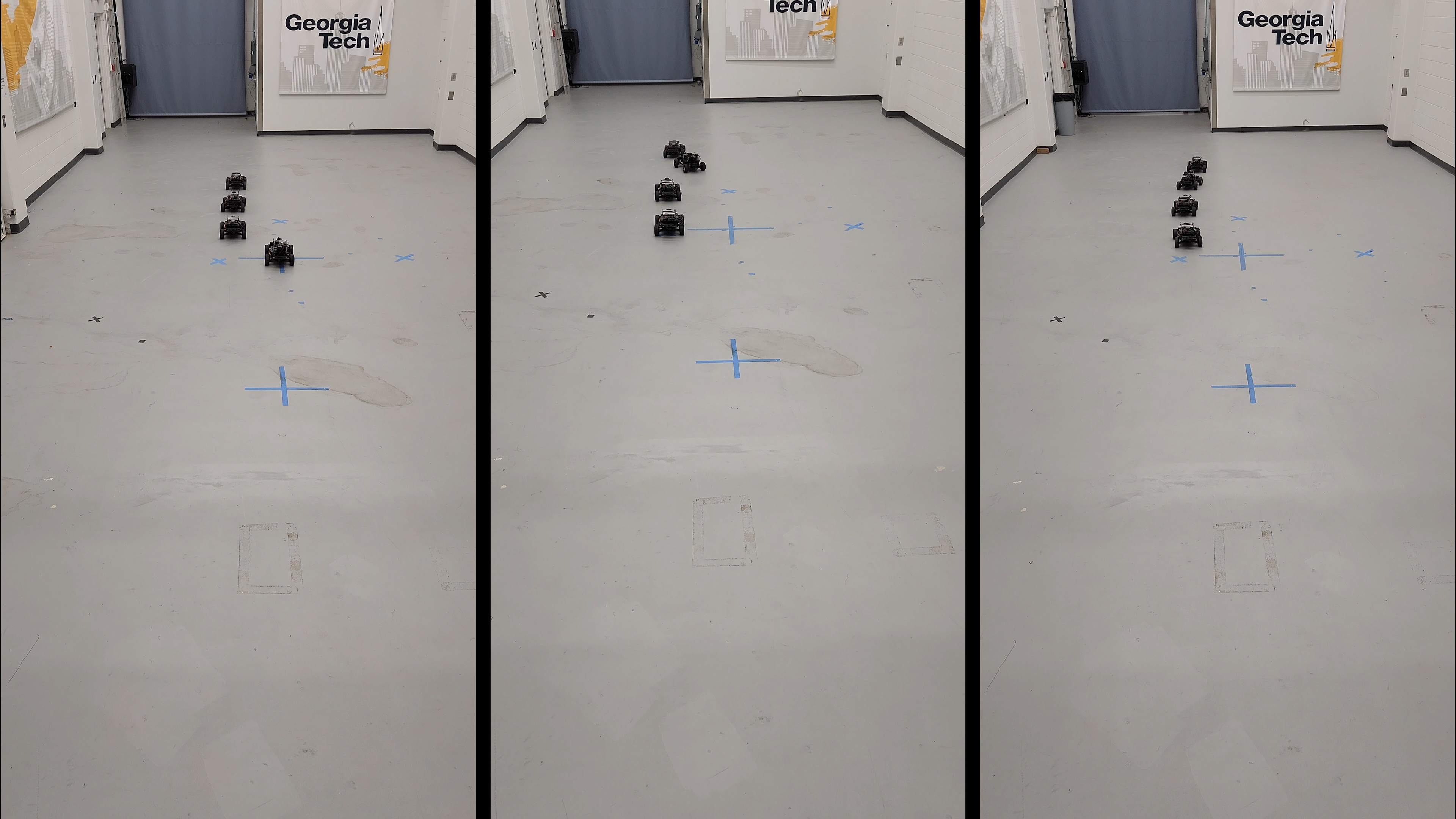}
    \end{subfigure}
    \begin{subfigure}{0.99\columnwidth}
        \includegraphics[width=\textwidth]{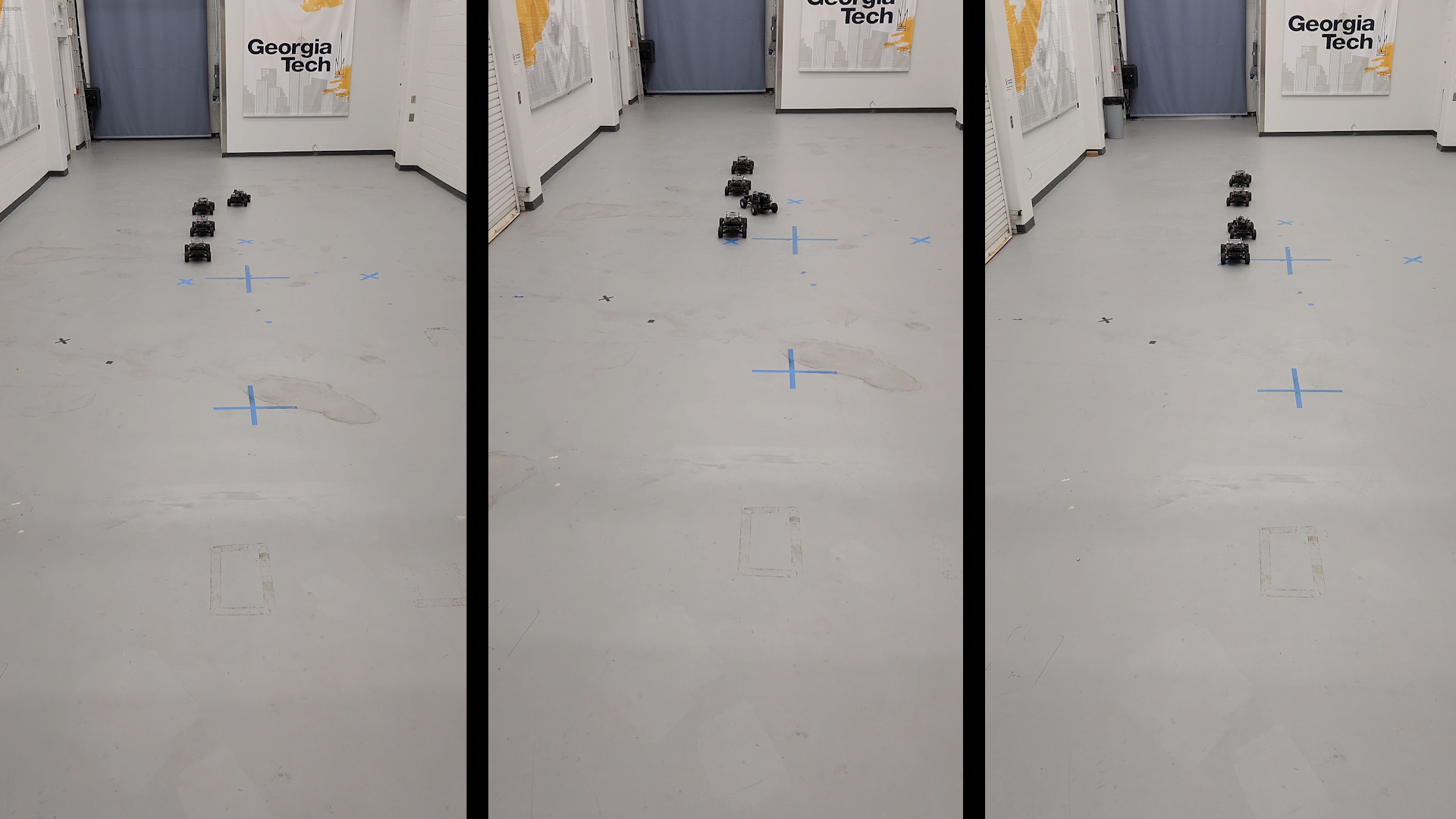}
    \end{subfigure}
    \caption{Snapshots of traffic merge experiment at 8-second increments, when the ego vehicle must overtake (left) or yield (right) to merge. Passive learning with EMPPI (left), active learning with DMPPI (middle) active learning with proposed DMPD (right). A video of several experiments is available at \url{https://youtu.be/Q_JdZuopGL4}.}
    \label{fig:experiment}
\end{figure*}

\section{Problem Formulation} \label{sec:problem_formulation}

Consider a nonlinear stochastic system given as
\begin{align}\label{eq:stoch_sys}
    x_{t+1} \sim f(x_{t}, u_{t}, \bar{\theta}),
\end{align}
where $x_{t} \in \Re^{n_x}$ and $u_{t} \in \Re^{n_u}$ denote the state and control action at time-step $t \in \mathbb{N}$, respectively, and $\bar{\theta} \in \Theta \subseteq \Re^{n_\theta}$ is a vector of unknown system parameters.
%
Let the cost function associated with \eqref{eq:stoch_sys}, for a horizon of length $N \in \mathbb{Z}^{+}$ be given as
\begin{align} \label{eq:horizon_cost}
    J(\mathrm{x}_{t:t+N}, \mathrm{u}_{t:t+N-1}) = \phi(x_{t+N}) + \sum_{k=t}^{t+N-1} \ell_{k}(x_{k}, u_{k}),
\end{align}
where, for $k \geq t$, $\mathrm{u}_{t:k} = \{u_{t}, u_{t+1}, \dots, u_{k}\}$, $\mathrm{x}_{t:k} = \{x_{t}, x_{t+1}, \dots, x_{k}\}$, $\phi(\cdot) : \Re^{n_x} \rightarrow \Re$ is the terminal cost, and $\ell_{k}(\cdot, \cdot) : \Re^{n_x} \times \Re^{n_u} \rightarrow \Re$ is the stage cost.
Our goal is to compute an optimal control policy, $u_{t} = \pi^{\ast}_{t}(x_{t})$, for \eqref{eq:stoch_sys} by solving the receding horizon problem
\begin{subequations} \label{prob:ideal}
    \begin{align}
        & \min_{\mathrm{u}_{t:t+N-1|t}} \Expectation \left[\phi(x_{t+N|t}) + \sum_{k=t}^{t+N-1} \ell_{k}(x_{k|t}, u_{k|t}) \right], \label{cost:ideal} \\
        & \text{subject to} \nonumber\\
        & x_{t|t} = x_{t}, \\
        & x_{k+1|t} \sim f(x_{k|t}, u_{k|t}, \bar{\theta}),
    \end{align}
\end{subequations}
where $\mathrm{u}_{t:t+N-1|t} = \{u_{t|t}, u_{t+1|t}, \dots, u_{t+N-1|t}\}$ and $u_{k|t}$ is the control action at future time $k$ planned at the present time $t$, and, similarly, $x_{k|t}$ is the future state at time $k$ predicted at the present time $t$.

However, since $\bar{\theta}$ is unknown, Problem~\eqref{prob:ideal} is ill-posed.
To resolve this discrepancy, let $\theta \sim b(\cdot)$ be an estimate of $\bar{\theta}$ sampled from the belief distribution $b(\cdot)$. 
Using a fixed belief distribution does not take advantage of information obtained during online operation to reduce uncertainty about $\theta$.
Therefore, we estimate $\theta$ online using Bayesian inference to refine the belief distribution by conditioning on the history of observations, namely,
\begin{align} \label{eq:bayesian_estimation}
    b_{t+1}(\theta) = b(\theta | \xi_{0}, \dots, \xi_{t}, x_{t+1}) ,
\end{align}
where $\xi_{t} = \{x_t, u_t\}$ is the observed information at time $t$.
We introduce the following lemma which allows for efficiently updating the conditional belief distribution in a recursive manner.
\begin{lemma}\label{lem:recursive_bayesian}
    The posterior belief distribution~\eqref{eq:bayesian_estimation} is given recursively by
    \begin{align} \label{eq:bayesian_recursive}
        b_{t+1}(\theta) \propto b_{t}(\theta) f(x_{t+1} | x_{t}, u_{t}, \theta),
    \end{align}
    where $b_{0}(\theta) = b(\theta)$ is the prior distribution, and $f(x_{t+1} | x_{t}, u_{t}, \theta)$ represents the probability density of $x_{t+1}$ under the dynamics conditioned on $x_{t}$, $u_{t}$, and $\theta$.
\end{lemma}
\begin{proof}
    Utilizing Bayes rule and the Markov property of system~\eqref{eq:stoch_sys}, we have 
    \begin{align*}
        b_{t+1}(\theta) &= b(\theta | \xi_{0}, \dots, \xi_{t}, x_{t+1}) \\
        b_{t+1}(\theta) &\propto b(x_{t+1} | \xi_{0}, \dots, \xi_{t}, \theta) b(\theta | \xi_{0}, \dots, \xi_{t}) \\
        b_{t+1}(\theta) &\propto f(x_{t+1} | x_{t}, u_{t}, \theta) b(\theta | \xi_{0}, \dots, \xi_{t-1}, x_{t})
    \end{align*}
    which is concisely given in a recursive form by \eqref{eq:bayesian_recursive}.
\end{proof}

Existing methods typically approximate $\pi^{\ast}_{t}(x_{t})$ using a certainty equivalence approach, which solves 
\begin{subequations} \label{prob:ce}
    \begin{align}
        & \min_{\mathrm{u}_{t:t+N-1|t}} \Expectation \left[\phi(x_{t+N|t}) + \sum_{k=t}^{t+N-1} \ell_{k}(x_{k|t}, u_{k|t}) \right], \\
        & \text{subject to} \nonumber\\
        & x_{t|t} = x_{t}, \\
        & x_{k+1|t} \sim f(x_{k|t}, u_{k|t}, \Expectation_{\theta \sim b_{t}}[\theta]),
    \end{align}
\end{subequations}
or a stochastic approach, which solves
\begin{subequations} \label{prob:stoch}
    \begin{align}
        & \min_{\mathrm{u}_{t:t+N-1|t}} \Expectation \left[\phi(x_{t+N|t}) + \sum_{k=t}^{t+N-1} \ell_{k}(x_{k|t}, u_{k|t}) \right], \\
        & \text{subject to} \nonumber\\
        & x_{t|t} = x_{t}, \quad \theta \sim b_{t}(\cdot), \label{const:stoch_dist} \\
        & x_{k+1|t} \sim f(x_{k|t}, u_{k|t}, \theta).
    \end{align}
\end{subequations}
However, neither of these methods account for the estimation process in the design of the control policy. Rather, they assume that $b_{k}(\cdot) = b_{t}(\cdot)$ for all $k = t, \dots, t+N-1$. 
As may be seen from \eqref{eq:bayesian_estimation}, not only is the assumption of a fixed time-invariant belief distribution incorrect, but the future belief distribution depends on the \emph{current control actions}. 
Therefore, there is a coupling between the control design and parameter estimation which ideally should be accounted for during optimization.
This is because the optimization problem incorporates counterfactual reasoning, enabling the autonomous vehicle to reason about how its actions will affect the likelihood of different outcomes of the actions of other vehicles.
This goes against the traditional ``predict-then-plan" paradigm, which is not suitable for highly interactive scenarios due to its static nature. 
In our formulation, rather than addressing them independently, prediction and planning are wholly coupled.

\section{Proposed Approach}

We wish to design an approximate control policy which reduces the optimality gap between Problems~\eqref{prob:ideal}~and~\eqref{prob:stoch} by accounting for the Bayesian estimation process in the design of the control policy $\pi^{\ast}_{t}(x_{t})$. 
Such a policy is referred to as a \emph{dual control} policy.

\subsection{Dual Control} \label{sec:proposed_dual_control}

Ideally, we would simply replace $b_{t}(\cdot)$ in \eqref{const:stoch_dist} with $b_{k}(\cdot)$ for $k = t, \dots, t+N-1$. 
However, as may be seen in \eqref{eq:bayesian_estimation}, $b_{k}(\cdot)$ depends on the unknown future realizations of $x_{k}$ for $k > t$.
Nonetheless, while the future state realizations may be unknown, we can still harness the information of the planned control actions.
Thus, we predict the future belief distributions by conditioning on the \emph{past observations} and \emph{planned actions}, as shown in Fig.~\ref{fig:active_learning}, so that the predicted belief distribution is given by
\begin{align}\label{eq:bayesian_approx}
    \hat{b}_{k+1|t}(\theta) &= b(\theta | \xi_{0:t}, u_{t+1:k}),
\end{align}
for $k = t, \dots, t+N-1$, where $\xi_{0:t} = \xi_{0}, \dots, \xi_{t}$.

\begin{theorem} \label{thm:predictive_belief}
    The predicted belief dynamics~\eqref{eq:bayesian_approx} are given by the recursive update
    \begin{align} \label{eq:bayesian_predictive}
        \hat{b}_{k+1|t}(\theta) \propto \hat{b}_{k|t}(\theta) \Expectation_{\bar{\theta} \sim b_{t}}[f(x_{k+1} | x_{t}, \mathrm{u}_{t:k}, \theta) | x_{t}, \mathrm{u}_{t:k}, \bar{\theta}],
    \end{align}
    where
    $\hat{b}_{t|t}(\theta) = b_{t}(\theta)$.
\end{theorem}
\begin{proof}
    From Lemma~\ref{lem:recursive_bayesian}, we have 
    \begin{align*}
        & b(\theta | \xi_{0}, \dots,\xi_{t}, \dots \xi_{k}, x_{k+1}) \propto b_{k}(\theta) f(x_{k+1} | x_{k}, u_{k}, \theta) \\
        & \propto b_{t}(\theta) f(x_{t+1} | x_{t}, u_{t}, \theta) \dots f(x_{k+1} | x_{k}, u_{k}, \theta),
    \end{align*}
    where the second line results from expanding the recursive definition of $b_{k}$.
    However, since the future states are not observed, we expand~\eqref{eq:bayesian_approx} using the predictive distribution marginalized over the unknown future observations given by
    \begin{align*}
        \hat{b}_{k+1|t}(\theta) & \propto b_{t}(\theta) \int \int \prod_{\ell=t}^{k} f(x_{\ell+1} | x_{\ell}, u_{\ell}, \theta) b_{t}(\theta) \mathrm{d}\theta \mathrm{d}x_{t+1} \\
        \hat{b}_{k+1|t}(\theta) & \propto b_{t}(\theta) \prod_{\ell=t}^{k} \int \int f(x_{\ell+1} | x_{\ell}, u_{\ell}, \theta) b_{t}(\theta) \mathrm{d}\theta \mathrm{d}x_{t+1}
    \end{align*}
    which is concisely expressed in a recursive form by~\eqref{eq:bayesian_predictive}.
\end{proof}

\begin{corollary} \label{cor:expected_belief}
    We can interpret Theorem~\ref{thm:predictive_belief} as
    \begin{align} \label{eq:expected_belief}
        \hat{b}_{k+1|t}(\theta) = \Expectation_{x_{t+1:k+1} \sim f(\cdot)}[b_{k+1}(\theta) | \xi_{0:t}, u_{t+1:k}],
    \end{align}
    for $k = t, \dots, t+N-1$,
    as an alternative to deriving~\eqref{eq:bayesian_predictive} from~\eqref{eq:bayesian_approx}.
\end{corollary}
\begin{proof}
    From~\eqref{eq:expected_belief}, and applying~\eqref{eq:bayesian_estimation}, we have 
    \begin{align*}
        \hat{b}_{k+1|t}(\theta) = \Expectation[b(\theta | \xi_{0:k}, x_{k+1}) | \xi_{0:t}, u_{t+1:k}].
    \end{align*}
    Then, similar to the arguments of Lemma~\ref{lem:recursive_bayesian},
    \begin{align*}
        & \hat{b}_{k+1|t}(\theta) \propto \Expectation[b(\xi_{0}, \dots, \xi_{k}, x_{k+1} | \theta) b(\theta) | \xi_{0:t}, \mathrm{u}_{t+1:k}] \\
        & \propto \Expectation[b(\theta) \prod_{\ell=0}^{k} f(x_{\ell+1} | x_{\ell}, u_{\ell}, \theta) | \xi_{0:t}, \mathrm{u}_{t+1:k}] \\
        & \propto b(\theta) \prod_{\ell=0}^{t-1} f(x_{\ell+1} | x_{\ell}, u_{\ell}, \theta) \\
        &\quad \times \Expectation[\prod_{\ell=t}^{k} f(x_{\ell+1} | x_{\ell}, u_{\ell}, \theta) | \xi_{0:t}, \mathrm{u}_{t+1:k}] \\
        & \propto b_{t}(\theta) \Expectation[\prod_{\ell=t}^{k} f(x_{\ell+1} | x_{\ell}, u_{\ell}, \theta) | \xi_{0:t}, \mathrm{u}_{t+1:k}] \\
        & \propto b_{t}(\theta) \prod_{\ell=t}^{k}  \Expectation[f(x_{\ell+1} | x_{\ell}, u_{\ell}, \theta) | x_{t}, \mathrm{u}_{t+1:k}],
    \end{align*}
    which leads to~\eqref{eq:bayesian_predictive} from Theorem~\ref{thm:predictive_belief}, with $\hat{b}_{t|t}(\theta) = b_{t}(\theta)$.
\end{proof}

Using Theorem~\ref{thm:predictive_belief}, we approximate the solution of Problem~\eqref{prob:ideal} by solving the optimal control problem
\begin{subequations} \label{prob:dual}
    \begin{align}
        & \min_{\mathrm{u}_{t:t+N-1|t}} \Expectation \left[\phi(x_{t+N|t}) + \sum_{k=t}^{t+N-1} \ell_{k}(x_{k|t}, u_{k|t}) \right], \\
        & \text{subject to} \nonumber\\
        & x_{t|t} = x_{t}, \quad \theta \sim \hat{b}_{k|t}(\cdot), \label{const:dual_dist} \\
        & \hat{b}_{k+1|t}(\theta) = b(\theta | \xi_{0:t}, u_{t+1:k}), \\
        & x_{k+1|t} \sim f(x_{k|t}, u_{k|t}, \theta),
    \end{align}
\end{subequations}
where we have the following proposition.
\begin{proposition}[\cite{knaup2024active}]
    The solution to Problem~\eqref{prob:dual} is a causal control policy. That is, the optimal control applied at time $t$ only depends on information available at or before time $t$.
\end{proposition}

\begin{figure}[th]
    \centering
    \includegraphics[width=0.99\linewidth, trim={2500 0 600 0}, clip]{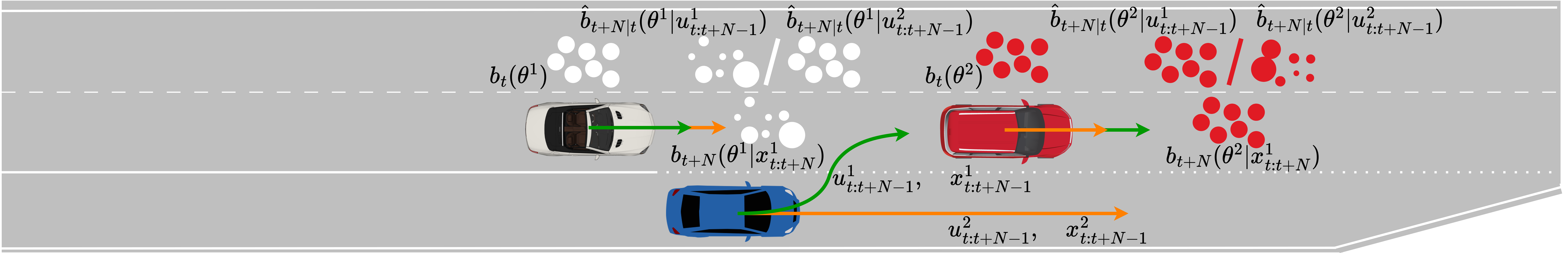}
    \caption{Belief prediction approach enabling active learning.}
    \label{fig:active_learning}
\end{figure}

\subsection{Model-based Diffusion} \label{sec:proposed_model_based_diffusion}

Define the cost corresponding to a solution of Problem~\eqref{prob:dual} as 
\begin{align}\label{eq:score_cost}
    & J_{\mathrm{DC}}(x_t, \mathrm{u}_{t:t+N-1}) \nonumber\\
    &= \Expectation_{\substack{\theta_{k} \sim \hat{b}_{k|t}(\cdot),\\ x_{k+1} \sim f(\cdot | x_{t}, \mathrm{u}_{t:k}, \theta_{k})}} \left[\phi(x_{t+N|t}) + \sum_{k=t}^{t+N-1} \ell_{k}(x_{k|t}, u_{k|t}) \right],
\end{align}
and let $u_{t:t+N-1|t}^{\ast} = \pi_{\mathrm{DC}}^{\ast}(x_{t})$ be the optimal solution of Problem~\eqref{prob:dual}, which minimizes~\eqref{eq:score_cost}, given $x_{t}$.
Problem~\eqref{prob:dual} is a causal stochastic optimal control problem that can be solved by a number of methods. 
However, in general, the cost function~\eqref{eq:score_cost} may be non-differentiable owing to 
penalty functions for collisions, for example.
Therefore, we opt to use a gradient-free sampling-based optimization scheme.

We define the probability distribution corresponding to~\eqref{eq:score_cost} as
\begin{align} \label{eq:score_dist}
    s^{0}(\mathrm{u}_{t:t+N-1} | x_{t}) \propto \exp{\left(-\frac{J_{\text{DC}}(x_t, \mathrm{u}_{t:t+N-1})}{\lambda} \right)},
\end{align}
where $s^{0}(\cdot)$ approaches a Dirac delta function as $\lambda \rightarrow 0$.
Equation~\eqref{eq:score_dist} may be interpreted as a probability distribution that puts highest probability on the optimal solution, $\pi_{\mathrm{DC}}^{\ast}(x_{t})$, of Problem~\eqref{prob:dual}.
Thus, solving problem~\eqref{prob:dual}, may alternatively be interpreted as sampling from \eqref{eq:score_dist} with small $\lambda$.

However, although \eqref{eq:score_dist} allows us to compute the probability density for a given solution, $\mathrm{u}_{t:t+N-1}$, it is not practical to directly sample from $s^0(\cdot | x_{t})$.
In order to generate approximate samples from \eqref{eq:score_dist}, we develop a novel multi-modal variant of the model-based diffusion algorithm presented in \cite{pan2024model}, tailored for the specific application of solving receding horizon problems.

\subsubsection{Generative Diffusion Models}
Diffusion addresses the problem in generative modeling of how to draw novel samples from a  distribution, $\tilde{y}^{0} \sim q^{0}(\cdot)$.
In the case of data-driven model-free diffusion, the distribution $q^{0}(\cdot)$ is unknown, but empirical samples from the unknown distribution, $\{y^{0, j}\}_{j=1}^{N_s} \sim q^{0}(y)$, are given.
In the case of analytical model-based diffusion, the probability density function of $q^{0}(\cdot)$ is available.
In either case, sampling from $q^{0}(\cdot)$ is intractable, so rather than sampling from $q^{0}(\cdot)$ directly, diffusion draws samples from a prior distribution $q^{N_d}$ and then maps them back to the desired distribution.

Several formulations of diffusion models exist, but the most general is that of stochastic differential equations (SDE).
In this case, the prior distribution is generated by corrupting the data samples with noise according to a hand-crafted SDE.
Let the transition dynamics from $q^{0}(\cdot)$ to $q^{N_d}(\cdot)$ be given by
\begin{align} \label{eq:discrete_forward_sde}
    y^{\tau+1} &= (I + \tilde{A}^{\tau}) y^{\tau} + B^{\tau} z^{\tau},
\end{align}
for $\tau = 0, 1, \dots, N_{d}-1$,
where $z^{\tau} \sim \mathcal{N}(0, I)$ 
and $y^{0} \sim q^{0}(\cdot)$.
The diffusion kernel parameters $\tilde{A}^{\tau}$ and $B^{\tau}$ are chosen by the designer to suite the needs of the application.
For example, the standard kernel used in denoising diffusion probabilistic models (DDPMs) is $\tilde{A}^{\tau} = \sqrt{1 - \beta_{\tau}^\mathrm{d}} I$, $B^{\tau} = \beta_{\tau}^{\mathrm{d}} I$, where $\beta_{\tau}^{\mathrm{d}} \in (0, 1)$, and the standard kernel used by Score-Based Generative Models (SGMs) is $\tilde{A}^{\tau} = I$, $B^{\tau} = \sigma_{\tau}^\mathrm{s} I$, where $\sigma_{\tau}^\mathrm{s} \in \Re$.

\begin{lemma} \label{lem:sde_ddpm_equivalence}
    The conditional distribution, $q^{\tau|0}(y^{\tau} | y^{0})$, resulting from the forward dynamics~\eqref{eq:discrete_forward_sde}, is given by
    \begin{align}
        q^{\tau|0}(y^{\tau} | y^{0}) = \mathcal{N}(\mu^{\tau}(y^0), \Sigma^{\tau}),
    \end{align}
    where
    \begin{subequations}
        \begin{align}
            \mu^{\tau+1}(y^0) &= \prod_{i=0}^{\tau} (I + \tilde{A}^{i}) y^{0},
        \end{align}
        which is equivalent to the dynamical system
        \begin{align}
            \mu^{\tau+1} &= (I + \tilde{A}^{\tau}) \mu^{\tau},
        \end{align}
        where $\mu^{0} = y^{0}$,
        and where
        \begin{align} \label{eq:sde_sigma_dynamics}
            \Sigma^{\tau+1} &= (I + \tilde{A}^{\tau}) \Sigma^{\tau} (I + \tilde{A}^{\tau})^{\top} + B^{\tau} B^{\tau^\top},
        \end{align}
        where $\Sigma_0 = 0$.
    \end{subequations}
\end{lemma}
\begin{proof}
    The result is easy to show by computing the conditional distributions using the linear dynamics~\eqref{eq:discrete_forward_sde}, and so the proof is omitted.
\end{proof}
\begin{remark}
    In practice, the forward process is often constructed such that $q^{N_d}(\cdot) = q(y^{N_d} | y^{0}) \rightarrow \mathcal{N}(0, I)$ as $N_d \rightarrow \infty$, for all $y^{0} \sim q^{0}(\cdot)$.
    The approximation $q^{N_d}(\cdot) \approx \mathcal{N}(0, I)$ then allows for sampling from a fixed prior distribution at inference time.
\end{remark}


Samples from the prior distribution are ``denoised" back to the data distribution by reversing the SDE.
\begin{proposition} [\cite{song2020score, yang2023diffusion}] \label{prop:reverse_sde}
%
    The forward diffusion process given by~\eqref{eq:discrete_forward_sde}
    may be reversed by the discrete SDE given by
    \begin{align}
        y^{\tau-1} = (I - \tilde{A}^{\tau}) y^{\tau} + B^{\tau} B^{{\tau}^\top} \nabla_{y} \log{q^{\tau}(y^{\tau})} + B^{\tau} z^{\tau},
    \end{align}
    or alternatively by the ODE
    \begin{align}
        y^{\tau-1} = (I - \tilde{A}^{\tau}) y^{\tau} + \frac{1}{2} B^{\tau} B^{{\tau}^\top} \nabla_{y} \log{q^{\tau}(y^{\tau})}.
    \end{align}
\end{proposition}
In general, the desired data distribution may be unknown, in which case $\nabla_{y} \log{q(y^{\tau})}$ must be learned from data.
However, in the case where the desired distribution is known (even if it cannot be sampled from), it is possible to compute $\nabla_{y} \log{q(y^{\tau})}$ directly, bypassing the training process.
\begin{proposition}[\cite{pan2024model}] \label{prop:score_function}
    The score function $\nabla_{y^{\tau}} \log{q^{\tau}(y^{\tau})}$ may be computed explicitly using $q^{0}(y^{0})$,
    \begin{align} \label{eq:score_function}
        &\nabla_{y} \log{q^{\tau}(y^{\tau})} \\
        & = - \Sigma^{\tau^{-1}} y^{\tau} + \Sigma^{\tau^{-1}} \frac{\int \mu^{\tau}(y^0) \mathcal{N}(y^{\tau} | \mu^{\tau}(y^0), \Sigma^{\tau}) q(y^{0}) dy^{0}} {\int \mathcal{N}(y^{\tau} | \mu^{\tau}(y^0), \Sigma^{\tau}) q(y^{0}) dy^{0}}, \nonumber
    \end{align}
    which may be approximated using importance sampling as
    \begin{align}\label{eq:is_score_function}
        &\nabla_{y} \log{q^{\tau}(y^{\tau})}
        \approx - \Sigma^{\tau^{-1}} y^{\tau} + \Sigma^{\tau^{-1}} \frac{\sum_{j=1}^{N_s} \mu^{\tau}(y^{0, j}) q(y^{0, j})} {\sum_{j=1}^{N_s} q(y^{0, j})},
    \end{align}
    where $y^{0, j} \sim \mathcal{N}(\bar{A}^{\tau^{-1}} y^{\tau, j}, \bar{A}^{\tau^\top} \Sigma^{\tau^{-1}} \bar{A}^{\tau})$ and $\bar{A}^{\tau} = \prod_{i=0}^{\tau} (I + \tilde{A}^{i})$, for all $j = 1, \dots, N_s$ and $i = N_d, \dots, 1$.
\end{proposition}
\begin{proof}
    The result may be derived as follows
    \begin{subequations}
    \begin{align}
        & \nabla_{y^{\tau}} \log{q^{\tau}(y^{\tau})}
        = \frac{\nabla_{y^{\tau}} q^{\tau}(y^{\tau})} {q^{\tau}(y^{\tau})} \label{eq:score_func_proof_chain_rule} \\
        & = \frac{\nabla_{y^{\tau}} \int q(y^{\tau}|y^{0}) q(y^{0}) dy^{0}} {\int q(y^{\tau}|y^{0}) q(y^{0}) dy^{0}} \label{eq:score_func_proof_bayes} \\
        & = \frac{\int \nabla_{y^{\tau}} q(y^{\tau}|y^{0}) q(y^{0}) dy^{0}} {\int q(y^{\tau}|y^{0}) q(y^{0}) dy^{0}} \label{eq:score_func_proof_order_operations} \\
        & = \frac{\int \nabla_{y^{\tau}} \mathcal{N}(y^{\tau} | \mu^{\tau}(y^0), \Sigma^{\tau}) q(y^{0}) dy^{0}} {\int \mathcal{N}(y^{\tau} | \mu^{\tau}(y^0), \Sigma^{\tau}) q(y^{0}) dy^{0}} \label{eq:score_func_proof_lemma} \\
        & = \frac{\int - \Sigma^{\tau^{-1}} (y^{\tau} - \mu^{\tau}(y^0)) \mathcal{N}(y^{\tau} | \mu^{\tau}(y^0), \Sigma^{\tau}) q(y^{0}) dy^{0}} {\int \mathcal{N}(y^{\tau} | \mu^{\tau}(y^0), \Sigma^{\tau}) q(y^{0}) dy^{0}} \label{eq:score_func_proof_grad_gaussian} \\
        & = - \Sigma^{\tau^{-1}} y^{\tau} \frac{\int \mathcal{N}(y^{\tau} | \mu^{\tau}(y^0), \Sigma^{\tau}) q(y^{0}) dy^{0}} {\int \mathcal{N}(y^{\tau} | \mu^{\tau}(y^0), \Sigma^{\tau}) q(y^{0}) dy^{0}} \nonumber \\
        & \quad + \frac{\int \Sigma^{\tau^{-1}} \mu^{\tau}(y^0) \mathcal{N}(y^{\tau} | \mu^{\tau}(y^0), \Sigma^{\tau}) q(y^{0}) dy^{0}} {\int \mathcal{N}(y^{\tau} | \mu^{\tau}(y^0), \Sigma^{\tau}) q(y^{0}) dy^{0}}, \label{eq:score_func_proof_expand}
    \end{align}
    \end{subequations}
    where in~\eqref{eq:score_func_proof_chain_rule} we apply the chain rule of derivation, 
    in~\eqref{eq:score_func_proof_bayes} we use Bayes' theorem to express $q^{\tau}(y^{\tau})$, 
    in~\eqref{eq:score_func_proof_order_operations} we move the derivative inside the integral since the gradient is being taken with respect to $y^{\tau}$ whereas the variable of integration is $y^{0}$,
    in~\eqref{eq:score_func_proof_lemma} we apply Lemma~\ref{lem:sde_ddpm_equivalence},
    in~\eqref{eq:score_func_proof_grad_gaussian} we take the gradient of the Gaussian distribution $\mathcal{N}(y^{\tau} | \mu^{\tau}(y^0), \Sigma^{\tau}) \propto \exp{(-\frac{1}{2} (y^{\tau} - \mu^{\tau}(y^0))^\top \Sigma^{i^{-1}} (y^{\tau} - \mu^{\tau}(y^0)))}$,
    and in~\eqref{eq:score_func_proof_expand} we expand and pull coefficients outside the integral,
    and then we simplify to obtain the result.
    The importance sampling approximation is obtained by reparameterizing 
    \begin{align*}
        & \mathcal{N}(y^{\tau} | \mu^{\tau}(y^0), \Sigma^{\tau}) \\
        &\propto \exp{(-\frac{1}{2} (y^{\tau} - \mu^{\tau}(y^0))^\top \Sigma^{\tau^{-1}} (y^{\tau} - \mu^{\tau}(y^0)))} \\
        & \propto \exp{(-\frac{1}{2} (\bar{A}^{\tau^{-1}}y^{\tau} - y^0)^\top \bar{A}^{\tau^\top} \Sigma^{\tau^{-1}} \bar{A}^{\tau} (\bar{A}^{\tau^{-1}} y^{\tau} - y^0))}
    \end{align*} 
    so as to sample $y^{0}$, rather than $y^{\tau}$, since $y^{0}$ is the variable of integration.
\end{proof}

\subsubsection{Model Predictive Diffusion}

\begin{figure*}[th]
    \centering
    \includegraphics[width=0.99\linewidth]{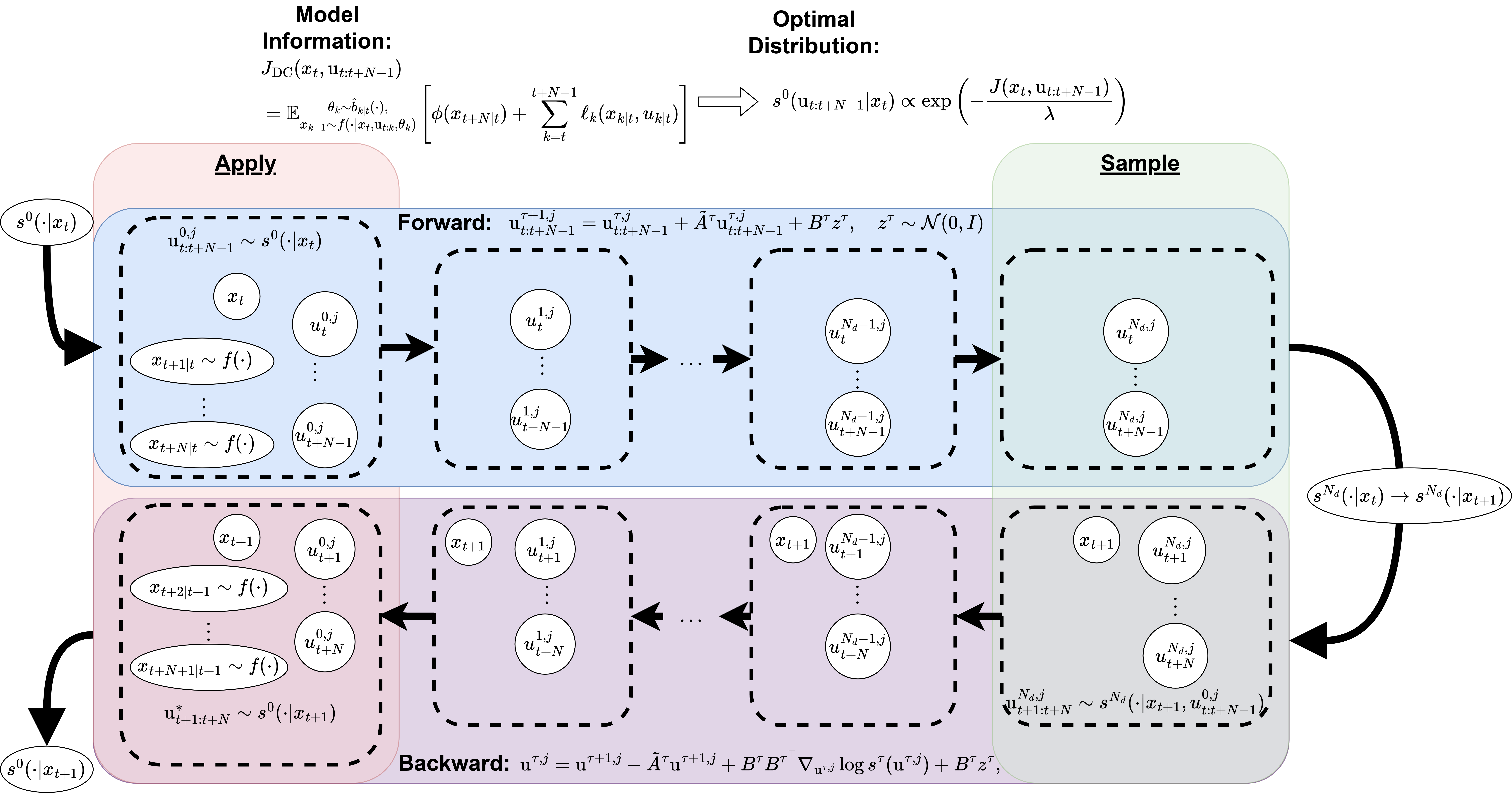}
    \caption{\textbf{Model Predictive Diffusion:} In the forward pass, $N_m$ samples from the previous time-step's optimal distribution $s^{0}(\cdot | x_{t})$ are corrupted with noise to form the dynamic prior distribution for the current time-step. In the backward pass, samples are drawn from the prior distribution $s^{N_d}(\cdot | x_{t+1})$ and are denoised to construct samples from  the updated optimal distribution $s^{0}(\cdot | x_{t+1})$. The samples $u_{t:t+N-1}^{\tau, j}$ are indexed by their step in the diffusion process $\tau$, their sample index (also referred to as mode of the prior distribution which is a mixture of $N_d$ Gaussians) $j$, and their time-step in real time $t$.}
    \label{fig:generative_diffusion}
\end{figure*}

Although the model-based diffusion algorithm is able to effectively solve optimization problems by generating samples from the optimal distribution, it has several limitations that hinder its application to model predictive control.
First, the iterative denoising process is computationally intensive as many denoising steps must be performed sequentially in order to move the samples towards the optimal value, and reducing the number of steps may lead to sub-optimal solutions.
Second, model-based diffusion assumes a fixed, unimodal Gaussian prior which does not allow for leveraging prior knowledge about the specific problem being solved.
Third, model-based diffusion does not exploit the online solution structure of model predictive control; rather, every time the problem must be solved, all previous information about the previous optimal solution is discarded.

We address these deficiencies by proposing a novel variant of model-based diffusion, tailored to the special structure of model predictive control.
Our proposed architecture reduces the number of steps required for the compute-intensive iterative denoising process by \emph{learning a multimodal dynamic prior distribution} so that samples may be drawn closer to the optimal distribution.
The multimodal dynamic prior is constructed by bootstrapping with multiple approximate local-optimal solutions to the MPC problem from the previous time-step, which we refer to as modes.
The modes are corrupted with noise to construct the prior as in regular diffusion; however, the forward diffusion process is truncated so that the noisy modes remain distinct, resulting in a multimodal prior which is dependent on the previous solution.
This key insight reduces the computational burden of the iterative denoising process, which must be performed sequentially, by setting up the prior as a close approximation of the new optimal distribution at the current time-step, assuming the solutions do not change rapidly between sampling times, a common assumption in the MPC literature.

Consider $N_m$ samples drawn from the optimal distribution $s^{0}(\cdot | x_{t})$, given by $\{\mathrm{u}_{t:t+N-1}^{0, j} \}_{j=1}^{N_m}$. 
These may be converted to noise following the forward diffusion process given by
\begin{align} \label{eq:proposed_markov_diffusion}
    &\mathrm{u}_{t:t+N-1}^{\tau+1, j} = (I + \tilde{A}^{\tau}) \mathrm{u}_{t:t+N-1}^{\tau, j} + B^{\tau} z^{\tau, j},
\end{align}
for $\tau = 0, \dots, N_d-1$,
where $z^{\tau, j} \sim \mathcal{N}(0, I)$,
resulting in 
\begin{align} \label{eq:proposed_forward_diffusion}
    &\mathrm{u}_{t:t+N-1}^{\tau+1, j} \sim 
    s^{\tau+1}(\cdot | \mathrm{u}_{t:t+N-1}^{0, j}, x_{t}) = \mathcal{N}(\cdot | \bar{A}^{\tau} \mathrm{u}_{t:t+N-1}^{0, j}, \Sigma^{\tau}),
\end{align}
for $j = 1, \dots, N_m$,
where $\bar{A}^{\tau} = \prod_{i=0}^{\tau} (I + \tilde{A}^{i})$, and $\Sigma^{\tau}$ is given in \eqref{eq:sde_sigma_dynamics}.
Control sequences $u_{t:t+N-1}^{\tau, j}$ are indexed by their step in the diffusion process $\tau$, their mode index of the prior distribution which is a mixture of $N_d$ Gaussians $j$, and their time-step in real time $t$.
Samples from the priors may then be denoised to generate samples from the optimal distribution $s^{0}(\cdot | x_{t})$ using the following result.

\begin{theorem} \label{thm:proposed_denoising}
    The diffusion process given by~\eqref{eq:proposed_markov_diffusion}, or, equivalently,~\eqref{eq:proposed_forward_diffusion}, may be reversed by the discrete SDE
    \begin{align} \label{eq:proposed_denoising_sde}
        \mathrm{u}_{t:t+N-1}^{\tau-1, j} &= (I - \tilde{A}^{\tau}) \mathrm{u}_{t:t+N-1}^{\tau, j} \nonumber\\
        &+ B^{\tau} B^{{\tau}^\top} \nabla_{\mathrm{u}_{t:t+N-1}^{\tau, j}} \log s^{\tau}(\mathrm{u}_{t:t+N-1}^{\tau, j}) + B^{\tau} z^{\tau, j},
    \end{align}
    or by the discrete ODE
    \begin{align} \label{eq:proposed_denoising_ode}
        \mathrm{u}_{t:t+N-1}^{\tau-1, j} &= (I - \tilde{A}^{\tau}) \mathrm{u}_{t:t+N-1}^{\tau, j} \nonumber\\
        &+ \frac{1}{2} B^{\tau} B^{{\tau}^\top} \nabla_{\mathrm{u}_{t:t+N-1}^{\tau, j}} \log s^{\tau}(\mathrm{u}_{t:t+N-1}^{\tau, j}).
    \end{align}
\end{theorem}
\begin{proof}
    The equivalence of~\eqref{eq:proposed_markov_diffusion}~and~\eqref{eq:proposed_forward_diffusion} may easily be seen from Lemma~\ref{lem:sde_ddpm_equivalence}.
    The results~\eqref{eq:proposed_denoising_sde}~and~\eqref{eq:proposed_denoising_ode} follow immediately from Proposition~\ref{prop:reverse_sde}.
\end{proof}

This gives rise to the proposed generative diffusion model.
Rather than allowing the diffusion process to proceed for enough steps to reach a fixed prior which is unconditional of the samples from the optimal distribution, as is usually the case, we instead limit the number of steps to reduce computations, producing $N_m$ \emph{unique} prior distributions.
For each of these modes, the denoising process in Theorem~\ref{thm:proposed_denoising} may be evaluated using the following result.
\begin{theorem} \label{thm:score_function}
    The score function $\nabla_{\mathrm{u}_{t:t+N-1}^{\tau, j}} \log{s^{\tau}(\mathrm{u}_{t:t+N-1}^{\tau, j})}$ may be computed explicitly using $s^{0}(\cdot | x_{t})$,
    \begin{align} \label{eq:score_function_control}
        &\nabla_{\mathrm{u}_{t:t+N-1}^{\tau, j}} \log{s^{\tau}(\mathrm{u}_{t:t+N-1}^{\tau, j})} \\
        & = - \Sigma^{\tau^{-1}} \mathrm{u}_{t:t+N-1}^{\tau, j} \nonumber\\
        &+ \Sigma^{\tau^{-1}} \frac{\int \bar{A}^{\tau} \mathrm{u}_{t:t+N-1}^{0, j} \mathcal{N}(\cdot | \bar{A}^{\tau} \mathrm{u}_{t:t+N-1}^{0, j}, \Sigma^{\tau}) s^{0}(\cdot) \mathrm{d}\mathrm{u}_{t:t+N-1}^{0, j}} {\int \mathcal{N}(\cdot | \bar{A}^{\tau} \mathrm{u}_{t:t+N-1}^{0, j}, \Sigma^{\tau}) s^{0}(\mathrm{u}_{t:t+N-1}^{0, j}) \mathrm{d}\mathrm{u}_{t:t+N-1}^{0, j}}, \nonumber
    \end{align}
    which may be approximated using importance sampling as
    \begin{align}\label{eq:is_score_function_control}
        &\nabla_{\mathrm{u}_{t:t+N-1}^{\tau, j}} \log{s^{\tau}(\mathrm{u}_{t:t+N-1}^{\tau, j})} \\
        &\approx - \Sigma^{\tau^{-1}} \mathrm{u}_{t:t+N-1}^{\tau, j} + \Sigma^{\tau^{-1}} \frac{\sum_{i=1}^{N_s} \bar{A}^{\tau} \mathrm{u}_{t:t+N-1}^{0, j, i} s^{0}(\mathrm{u}_{t:t+N-1}^{0, j, i})} {\sum_{i=1}^{N_s} s^{0}(\mathrm{u}_{t:t+N-1}^{0, j, i})}, \nonumber
    \end{align}
    where 
    \begin{align} \label{eq:importance_sampling}
        &\mathrm{u}_{t:t+N-1}^{0, j, i} \sim \mathcal{N}(\bar{A}^{\tau^{-1}} \mathrm{u}_{t:t+N-1}^{\tau, j}, \bar{A}^{\tau^\top} \Sigma^{\tau^{-1}} \bar{A}^{\tau})
    \end{align}
    for all $i = 1, \dots, N_s$ and $\tau = N_d, \dots, 1$, where $i$ is an additional index representing multiples samples from each mode $j$.
\end{theorem}
\begin{proof}
    The result follows immediately from Proposition~\ref{prop:score_function}, and thus the proof is omitted.
\end{proof}

The key issue is then how to generate the samples $\{\mathrm{u}_{t:t+N-1}^{0, j} \sim s^{0}(\cdot | x_{t})\}_{j=1}^{N_m}$ from the optimal distribution.
We exploit the receding horizon structure of MPC by utilizing the common assumption that the optimal solution is similar between consecutive time-steps, and employ the approximation $s^{0}(\cdot | x_{t+1}) \approx s^{0}(\cdot | x_{t})$.
Thus, the previous solutions are used to bootstrap the diffusion process at the next sampling time using 
\begin{align} \label{eq:receding_horizon_prior}
    \mathrm{u}_{t+1:t+N}^{N_d, j} \sim \mathcal{N}(\bar{A}^{N_d-1} \mathrm{u}_{t:t+N-1}^{0, j}, \Sigma^{N_d-1}).
\end{align}
This process is summarized by Algorithm~\ref{alg:RHD} and shown in Fig.~\ref{fig:generative_diffusion}.
The key idea is to use the previous Model Predictive Diffusion (MPD) solution to construct a prior distribution made up of a mixture of Gaussians with the previous solution as the modes.
Samples from the Gaussian mixture prior are then denoised to derive multiple samples of the optimal target distribution, following the standard model-based diffusion procedure.
The optimal control to apply is then selected as the optimal of these samples, while all samples are used to bootstrap the next iteration of the MPD algorithm and construct a new prior for the next time-step.

\begin{algorithm}[h]
\footnotesize
\caption{Model Predictive Generative Diffusion}
\label{alg:RHD}
\begin{algorithmic}[1]
\Require Number of modes, number of diffusion steps: $N_m$, $N_d$
\Require Optimal probability distribution function $s^{0}(\cdot)$
\Require Initial control samples guess: $\{\mathrm{u}^{0, j}_{0:N-1}\}_{j=1}^{N_m}$
\Require Forward diffusion dynamics parameters: $\{\tilde{A}^{\tau}$, $B^{\tau} \}_{\tau=1}^{N_d}$
\For{$t = 1, 2, \dots$}
    \For{$j = 1, \dots, N_m$}{ \textbf{in parallel}}
        \State Construct dynamic prior distribution using forward SDE \eqref{eq:proposed_forward_diffusion}
        \State Sample from prior distribution using~\eqref{eq:receding_horizon_prior} 
        \For{$\tau = N_d-1, \dots, 0$}
            \State Calculate $\nabla_{\mathrm{u}_{t:t+N-1}^{\tau, j}} \log{s^{\tau}(\cdot)}$ using Theorem~\ref{thm:score_function} 
            \State Take one step of the backwards dynamics using Theorem~\ref{thm:proposed_denoising} 
        \EndFor
    \EndFor
    \State Select optimal $\mathrm{u}_{t:t+N-1}^{\ast}$ from $\{\mathrm{u}_{t:t+N-1}^{0, j} \}_{j=1}^{N_m}$
\EndFor
\end{algorithmic}
\end{algorithm}

\begin{remark} \label{lem:mppi}
    Observe that when $N_d = 1$, $N_m = 1$, and $\tilde{A}^{\tau} = I$,~\eqref{eq:proposed_denoising_ode} reduces to 
    \begin{align} \label{eq:mppi}
        &\mathrm{u}^{0, 1}_{t:t+N-1} =\frac{\sum_{i=1}^{N_s} \mathrm{u}_{t:t+N-1}^{0, 1, i} s^{0}(\mathrm{u}_{t:t+N-1}^{0, 1, i})} {\sum_{i=1}^{N_s} s^{0}(\mathrm{u}_{t:t+N-1}^{0, 1, i})},
    \end{align}
    where $\mathrm{u}_{t:t+N-1}^{0, 1, i} \sim \mathcal{N}(\mathrm{u}_{t:t+N-1}^{N_d, 1}, \frac{1}{2} B^{N_d} B^{{N_d}^{\top}})$.
    Equation~\eqref{eq:mppi} is the importance sampling policy employed by model predictive path integral control.
\end{remark}
\begin{proof}
    Utilizing~\eqref{eq:proposed_denoising_ode}~and~\eqref{eq:is_score_function_control}, we have
    \begin{align*}
        &\mathrm{u}_{t:t+N-1}^{0, 1} = (I - \tilde{A}^{N_d}) \mathrm{u}_{t:t+N-1}^{N_d, 1} \nonumber\\
        &\quad + \frac{1}{2} B^{N_d} B^{{N_d}^\top} \nabla_{\mathrm{u}_{t:t+N-1}^{N_d, 1}} \log s^{N_d}(\mathrm{u}_{t:t+N-1}^{N_d, 1}) \\
        &\mathrm{u}_{t:t+N-1}^{0, 1} = (I - \tilde{A}^{N_d}) \mathrm{u}_{t:t+N-1}^{N_d, 1} \\
        &\quad + \frac{1}{2} B^{N_d} B^{{N_d}^\top} \Big[- \Sigma^{N_d^{-1}} \mathrm{u}_{t:t+N-1}^{N_d, 1} \\
        &\quad + \Sigma^{N_d^{-1}} \frac{\sum_{i=1}^{N_s} \bar{A}^{N_d} \mathrm{u}_{t:t+N-1}^{0, 1, i} s^{0}(\mathrm{u}_{t:t+N-1}^{0, 1, i})} {\sum_{i=1}^{N_s} s^{0}(\mathrm{u}_{t:t+N-1}^{0, 1, i})} \Big] \\
        &\mathrm{u}^{0, 1}_{t:t+N-1} = \mathrm{u}^{N_d, 1}_{t:t+N-1} \\
        & \quad + \frac{1}{2} B^{N_d} B^{{N_d}^\top} \Big[ - (B^{N_d} B^{{N_d}^\top})^{-1} \mathrm{u}^{N_d, 1}_{t:t+N-1} \nonumber\\
        & \quad + (B^{N_d} B^{{N_d}^\top})^{-1} \frac{\sum_{i=1}^{N_s} \mathrm{u}^{0, 1, i}_{t:t+N-1} s^{0}(\mathrm{u}^{0, 1, i}_{t:t+N-1})} {\sum_{i=1}^{N_s} s^{0}(\mathrm{u}^{0, 1, i}_{t:t+N-1})} \Big], \\ 
        &\mathrm{u}^{0, 1}_{t:t+N-1} = \mathrm{u}^{N_d, 1}_{t:t+N-1}  - \frac{1}{2} \mathrm{u}^{N_d, 1}_{t:t+N-1} \nonumber\\
        & \quad + \frac{1}{2} \frac{\sum_{i=1}^{N_s} \mathrm{u}^{0, 1, i}_{t:t+N-1} s^{0}(\mathrm{u}^{0, 1, i}_{t:t+N-1})} {\sum_{i=1}^{N_s} s^{0}(\mathrm{u}^{0, 1, i}_{t:t+N-1})}, \\
    \end{align*}
    where $\mathrm{u}^{0, 1, i}_{t:t+N-1} \sim \mathcal{N}(\mathrm{u}_{t:t+N-1}^{N_d, 1}, B^{N_d} B^{{N_d}^{\top}})$.
    Thus 
    \begin{align*}
        &\mathrm{u}^{0, 1}_{t:t+N-1} = \mathrm{u}^{N_d, 1}_{t:t+N-1}  - \frac{1}{2} \mathrm{u}^{N_d, 1}_{t:t+N-1} + \frac{1}{2}  \mathrm{u}_{t:t+N-1}^{N_d, 1} \nonumber\\
        & \quad + \frac{1}{2} \frac{\sum_{i=1}^{N_s} \mathrm{\tilde{u}}^{0, 1, i}_{t:t+N-1} s^{0}(\mathrm{\tilde{u}}^{0, 1, i}_{t:t+N-1})} {\sum_{i=1}^{N_s} s^{0}(\mathrm{\tilde{u}}^{0, 1, i}_{t:t+N-1})}, \\
        &\mathrm{u}^{0, 1}_{t:t+N-1} = \mathrm{u}^{N_d, 1}_{t:t+N-1} + \frac{1}{2} \frac{\sum_{i=1}^{N_s} \mathrm{\tilde{u}}^{0, 1, i}_{t:t+N-1} s^{0}(\mathrm{\tilde{u}}^{0, 1, i}_{t:t+N-1})} {\sum_{i=1}^{N_s} s^{0}(\mathrm{\tilde{u}}^{0, 1, i}_{t:t+N-1})},
    \end{align*}
    where $\mathrm{\tilde{u}}^{0, 1, i}_{t:t+N-1} \sim \mathcal{N}(0, B^{N_d} B^{{N_d}^{\top}})$,
    which reduces to the result~\eqref{eq:mppi}.
\end{proof}
Thus, the dual MPPI algorithm employed in our prior work \cite{knaup2024active} may be interpreted as a special case of the proposed approach, since (dual) MPPI uses the sampling scheme given by \eqref{eq:mppi}. 
The proposed model predictive diffusion algorithm is more flexible, however, since it allows for multiple sampling steps (when $N_{d} > 1$) and sampling from a multi-modal distribution (when $N_m > 1$).

\subsection{Sampling Approximation} \label{sec:proposed_sampling_approximation}

We utilize sampling-based approximations to update the belief distribution~\eqref{eq:bayesian_recursive}, conditional predicted belief distributions~\eqref{eq:bayesian_predictive}, and expected cost~\eqref{prob:dual}. 
Additionally, we utilize the importance sampling approximation suggested by Theorem~\ref{thm:score_function} to evaluate the gradient steps of the denoising process.

The belief distribution is learned online using a particle filter. 
To this end, we draw $N_p$ samples from $b(\cdot)$ according to 
\begin{align} \label{eq:particle_sampling}
    \theta^{j} \sim b(\cdot), \quad j = 1, \dots, N_p,
\end{align}
and initialize the corresponding weights to $\omega^{j}_{0} = 1 / N_p$.
The belief distribution is then approximated online by updating the weights according to
\begin{align} \label{eq:particle_weight_update}
    \omega^{j}_{t+1} \propto \omega^{j}_{t} f(x_{t+1} | x_{t}, u_{t}, \theta^{j}),
\end{align}
which is the particle approximation of~\eqref{eq:bayesian_recursive}.
\begin{remark}
    Various resampling schemes may be added to \eqref{eq:particle_sampling}-\eqref{eq:particle_weight_update} to enrich the sampled parameters. However, this is an implementation consideration rather than a theoretical one, and as noted in \cite{knaup2024active}, we found it unnecessary for the application considered.
\end{remark}

The predicted belief distribution is propagated using a similar particle approximation.
The predicted particles are resampled using the current belief distribution
\begin{align} \label{eq:predicted_resampling}
    \hat{\theta}^{j}_{t} \sim b_{t}(\cdot), \quad j = 1, \dots, \hat{N}_p,
\end{align}
and the predicted weights are initialized as $\hat{\omega}^{j}_{t|t} = 1 / \hat{N}_p$.
The weights are then forward predicted using a particle approximation of~\eqref{eq:bayesian_predictive}, given by
\begin{align} \label{eq:predicted_weight_update}
    \hat{\omega}^{j}_{k+1|t} \propto \hat{\omega}^{j}_{k|t} f(\Expectation_{\bar{\theta} \sim b_{t}}[x_{k+1} | x_{t}, u_{t:k}, \bar{\theta}] | x_{t}, \mathrm{u}_{t:k}, \hat{\theta}_{t}^{j}) ,
\end{align}
for $k = t, t+1, \dots, t+N-1$,
where we have moved the expectation inside the pdf to reduce the number of costly evaluations of the pdf,
which may be interpreted as conditioning the predicted belief distribution on the expected observation rather than the expected update. 

Finally, the cost of Problem~\eqref{prob:dual} is evaluated using the sample average approximation given by 
\begin{subequations} \label{prob:saa}
    \begin{align}
        & \min_{\mathrm{u}_{t:t+N-1|t}} \sum_{j=1}^{\hat{N}_p} \left[\phi(x_{t+N|t}^{j}) \hat{\omega}_{t+N|t}^{j} + \sum_{k=t}^{t+N-1} \ell_{k}(x_{k|t}^{j}, u_{k|t}) \hat{\omega}_{k|t}^{j} \right], \label{cost:saa} \\
        & \text{subject to} \nonumber\\
        & x_{t|t}^{j} = x_{t}, \quad \hat{\theta}^{j}_{t} \sim b_{t}(\cdot), \quad \hat{\omega}_{t|t}^{j} = 1/\hat{N}_p \quad j = 1, \dots, \hat{N}_p, \\
        & \hat{\omega}^{j}_{k+1|t} \propto \hat{\omega}^{j}_{k|t} f(\frac{1}{\hat{N}_p} \sum_{j=1}^{\hat{N}_{p}} x_{k+1|t}^{j} | x_{t}, \mathrm{u}_{t:k|t}, \hat{\theta}_{t}^{j}), \label{const:saa_weights} \\
        & x_{k+1|t}^{j} \sim f(x_{k|t}^{j}, u_{k|t}, \hat{\theta}^{j}). \label{const:saa_dynamics}
    \end{align}
\end{subequations}
Associated with problem~\eqref{prob:saa}, we have the following theorem.
\begin{theorem}[\cite{knaup2024active}]
    The solution to problem~\eqref{prob:saa} preserves the dual control effect \cite{hu2022active}, that is, the planned control actions affect the entropy of the predicted future belief distribution.
\end{theorem}
\begin{proof}
    The proof follows from \eqref{eq:predicted_weight_update}, in which the planned control sequence $u_{t:k|t}$ affects the weights $\hat{\omega}^{j}_{k+1|t}$ of the categorical distribution over $\{\theta^j\}_{j=1}^{\hat{N}_p}$.
\end{proof}

Finally, the proposed dual model predictive diffusion approach is summarized in Algorithm~\ref{alg:DMPD}.
Although the sampling approximations used here are very common in the robotics literature, it should be noted that the accuracy of these approximations relies on the ability to draw a sufficient number of samples.
To this end, our implementation in Section~\ref{sec:interactive_autonomous_driving} relies on parallel GPU-enabled processing using Jax~\cite{jax2018github}.

\begin{algorithm}[h]
\footnotesize
\caption{Dual Model Predictive Diffusion}
\label{alg:DMPD}
\begin{algorithmic}[1]
\Require Number of particles, downsampled particle count, number of modes, number of diffusion steps, importance sampling size: $N_p$, $\hat{N}_p$, $N_m$, $N_d$, $N_s$
\Require Prior parameter belief distribution $b(\theta)$, dynamics $f(\cdot, \cdot, \cdot)$, cost function $J_{\text{DC}}(\cdot, \cdot)$
\Require Prior control samples/guess: $\{\mathrm{u}^{0, j}_{-1}\}_{j=1}^{N_m}$
\Require Forward diffusion dynamics parameters: $\{\tilde{A}^{\tau}$, $B^{\tau} \}_{\tau=1}^{N_d}$
\State Initialize belief distribution $b_{0}(\cdot) = b(\cdot)$, sample parameters using~\eqref{eq:particle_sampling}, and initialize weights $\omega^{j}_{0} = 1 / N_p$
\For{$t = 0, 1, \dots$}
    \State Update belief weights using~\eqref{eq:particle_weight_update}
    \State Sample predicted parameters using~\eqref{eq:predicted_resampling} and initialize predicted weights $\hat{\omega}^{j}_{t|t} = 1 / \hat{N}_p$
    \State Sample $N_m$ control sequences from the prior~\eqref{eq:receding_horizon_prior}
    \For{$\tau = N_d-1, \dots, 0$}
        \State Sample control sequences $\{u_{t:t+N-1}^{\tau, j, i} \}_{j=1, i=1}^{N_m, N_{s}}$ from~\eqref{eq:importance_sampling}
        \State Propagate the dynamics for each $u_{t:t+N-1}^{\tau, j, i}$
        using~\eqref{const:saa_dynamics}
        \State Propagate the weights 
        according to~\eqref{const:saa_weights}
        \State Compute the cost 
        according to~\eqref{cost:saa}
        \State Evaluate the optimal pdf
        according to~\eqref{eq:score_dist}
        \State Estimate the score function using~\eqref{eq:is_score_function_control}
        \State Take one step of the backwards dynamics using~\eqref{eq:proposed_denoising_ode}
    \EndFor
    \State Select and apply optimal $\mathrm{u}_{t:t+N-1}^{\ast}$ from $\{\mathrm{u}_{t:t+N-1}^{0, j} \}_{j=1}^{N_m}$
\EndFor
\end{algorithmic}
\end{algorithm}

\section{Interactive Autonomous Driving} \label{sec:interactive_autonomous_driving}

We evaluate the proposed method in a challenging interactive autonomous driving application where an autonomous ego vehicle must successfully complete a lane change/merge with one of several non-cooperative traffic vehicles.
In particular, we consider congested driving conditions in which inter-vehicle interaction is necessary to successfully complete the merge before reaching the end of the merge window. 
Moreover, we consider traffic vehicles having unknown heterogeneous driving behaviors, with varying levels of friendliness vs. aggressive behavior, so that the ego vehicle must successfully identify the other vehicles' behaviors online in order to complete the merge. 

We consider a scenario shown in Fig.~\ref{fig:merge_scenario} with one ego vehicle in the merge lane and $n_v$ traffic vehicles in the main lane.
The state of the vehicles at time $t$ is given by
$x_{t}^{i} = [v_{t}^{i}, \Psi_{t}^{i}, X_{t}^{i}, Y_{t}^{i}]^{\top}$
where $i = 0, 1, \dots, n_v$,
where 0 is the index of the ego vehicle,
1 is the index of the rear-most traffic vehicle,
$n_v$ is the index of the lead traffic vehicle,
and where $v$ is the vehicle's longitudinal velocity,
$\Psi$ is the heading angle,
and $X$ and $Y$ are the Cartesian position.
The state of all vehicles is given by the stacked vector 
$x_{t} = [x_{t}^{0 \top}, x_{t}^{1 \top}, \dots, x_{t}^{n_v \top} ]^{\top}$,
and it is assumed the state $x_{t}$ is fully observable to all vehicles.
\begin{figure}[ht]
    \centering
    \includegraphics[width=0.99\linewidth]{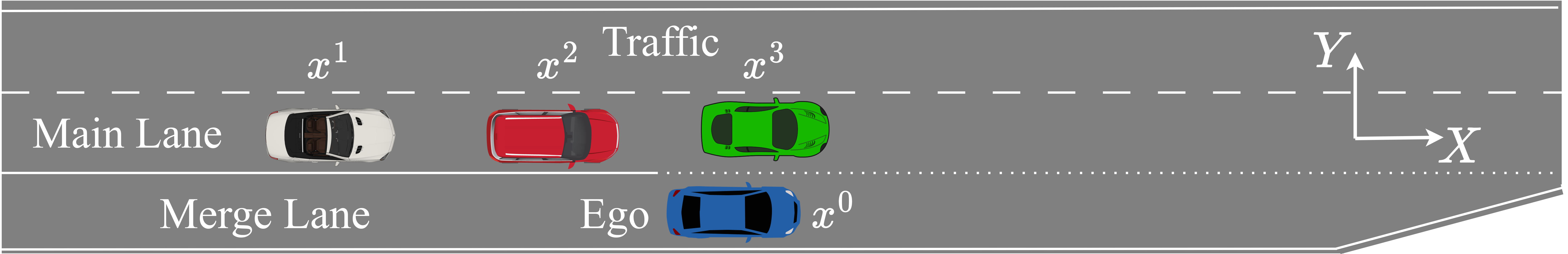}
    \caption{Merge scenario.}
    \label{fig:merge_scenario}
\end{figure}

In contrast to our previous work \cite{knaup2024active}, in which we used Dual MPPI only for generating a motion plan which was then followed by a low-level controller, in this paper, we utilize the proposed approach as an integrated motion planner and control algorithm.
Consequently, we consider the decision variables of the ego vehicle to be the steering angle $\delta_{t}^{0}$ and the longitudinal acceleration $a_{t}^{0}$, so that the control action is given by
$u_{t} = [a_{t}^{0}, \delta_{t}^{0}]^{\top}$.
The vehicle dynamics are modeled using the kinematic bicycle model \cite{kong2015kinematic}, given by 
\begin{subequations}
    \begin{align}
        \dot{v}_{t}^{i} &= a_{t}^{i}, \\
        \dot{\Psi}_{t}^{i} &= \sin(\beta_{t}^{i}) v_{t}^{i} / L_{r}^{i}, \\
        \dot{X}_{t}^{i} &= v_{t}^{i} \cos(\psi_{t}^{i} + \beta_{t}^{i}), \\
        \dot{Y}_{t}^{i} &= v_{t}^{i} \sin(\psi_{t}^{i} + \beta_{t}^{i}),
    \end{align}
\end{subequations}
where $\beta_{t}^{i} = \arctan(\frac{L_{r}^{i}}{L_{f}^{i} + L_{r}^{i}} \tan(\delta_{t}^{i}))$, and where $L_{f}$ and $L_{r}$ are the distances between the center of mass and the front and rear axles, respectively, which are assumed known for all vehicles.
As in our previous work \cite{knaup2024active}, we model the acceleration of the traffic vehicles using the merge-reactive intelligent driver model (MR-IDM) \cite{holley2023mr}, given by $a_{t}^{i} = \rho(x_{t}, \theta^{i})$, for $i=1, \dots, n_v$, where $\theta^{i}$ are the $i^{th}$ vehicle's driving behavior parameters (which is the vector of MR-IDM parameters), and where the equations for $\rho(\cdot, \cdot)$ may be found in \cite{holley2023mr}.
Therefore, the dynamics model is given by
\begin{align}
    f(x_{t}, u_{t}, \theta) &= x_{t} + [\dot{x}_{t}^{0}(x_{t}^{0}, \sigma(u_{t}))^{\top}, \dot{x}_{t}^{1}(x_{t}^{1}, \rho(x_{t}, \theta^{1}))^{\top}, \dots, \nonumber\\ 
    &\quad \dot{x}_{t}^{n_v}(x_{t}^{n_v}, \rho(x_{t}, \theta^{n_v}))^{\top} ]^{\top} \Delta t + w_{t},
\end{align}
where $\theta = [\theta^{1 \top}, \dots, \theta^{n_v \top} ]^{\top}$, 
$\sigma(\cdot)$ is a clamping function to impose the constraints $a_{\mathrm{min}} \leq a_{t} \leq a_{\mathrm{max}}$, $\delta_{\mathrm{min}} \leq \delta_{t} \leq a_{\mathrm{max}}$,
$\dot{x}_{t}^{i}(\cdot, \cdot) = [\dot{v}_{t}^{i \top}, \dot{\Psi}_{t}^{i \top}, \dot{X}_{t}^{i \top}, \dot{Y}_{t}^{i \top}]^{\top}$, 
$\Delta t$ is the integration time-step for the discrete dynamics, 
$w_{t} \sim \mathcal{N}(0_{n_x \times n_x}, \Sigma_w)$, 
and with $\Sigma_w \succ 0_{n_x \times n_x}$.
The unknown parameter vectors for all vehicles $\theta$ are estimated online using a particle filter as described in Section \ref{sec:proposed_sampling_approximation}, with the prior distribution described in~\cite{knaup2024active}.
As described in Algorithm~\ref{alg:DMPD} and Section~\ref{sec:proposed_sampling_approximation}, the current belief state from the particle filter is used to initialize the dual control problem at every time-step.

The cost is given by
\begin{subequations}
\begin{align}
    \ell_k(x_k) &= (x_k^{0} - x^g)^\top Q (\star) + u_{k}^{\top} R u_{k} + \ell^\text{pen}(x_k), \\
    \phi(x_{t+N}) &= (x_{t+N}^{0} - x^g)^\top Q_f (\star) + \ell^\text{pen}(x_{t+N}), \\
    \ell^\text{pen}(x_k) &= (\mathbf{1}^{\text{coll}}(x_k) + \mathbf{1}^\text{road}(x_k^{0}) + \mathbf{1}^\text{inval}(x_k)) Q_{\mathrm{pen}},
\end{align}
\end{subequations}
for $k=\{t, t+1,\dots,t+N-1 \}$,
where $x^g = [v^g, 0, 0, 0]^\top$ is the goal state, $Q = \text{diag}(Q_{v}, Q_{\Psi}, Q_{X}, Q_{Y}) \succeq 0_{n_x \times n_x}$ is the state cost matrix, $Q_f = \text{diag}(Q_{v_f}, Q_{\Psi_f}, Q_{X_f}, Q_{Y_f}) \succeq 0_{n_x \times n_x}$ is the terminal cost matrix, $R = \text{diag}(R_{a}, R_{\delta}) \succ 0_{n_u \times n_u}$, $Q_{\mathrm{pen}} \in \Re$ is the violation penalty coefficient, and $\mathbf{1}^{\text{coll}}(x_k)$, $\mathbf{1}^\text{road}(x_k)$, $\mathbf{1}^\text{inval}(x_k)$ are the indicator functions for a collision with another vehicle, violating the road boundaries, or an improper merge (not between two vehicles), respectively.\footnote{
As suggested in \cite{williams2018information}, control constraints are incorporated through clamping functions in the dynamics and task-related state constraints are incorporated through weighted indicator penalty functions in the cost.}

\subsection{Traffic Merge Experiment} \label{sec:traffic_merge_experiment}

We implemented the interactive autonomous driving scenario using a variation of the F1-Tenth autonomous vehicle platform, shown in Fig.~\ref{fig:racecar}.
The vehicles feature onboard compute using NVIDIA Jetson TX2s for the traffic vehicles and a NVIDIA Jetson Orin Nano for the ego vehicle.
The computational constraints of the Jetson Orin Nano require an efficient algorithm and highlight the computational efficiency of the proposed dual model predictive diffusion approach, which we ran at 10 Hz.
The experiments are carried out in the Georgia Tech Indoor Flight Lab (IFL), which provides a large open space to act as the ``highway" for the merge scenario.

The traffic vehicles were driven by an adaptive cruise control system which attempts to maintain a target distance from the leading vehicle, initialized to be less than the minimum distance needed to allow a merge between the vehicles.
One of the two traffic vehicles, designated the ``friendly" driver for that trial, were semi-human-controlled where the human would manually increase the following distance if and only if the ego vehicle was attempting to merge between the friendly car and its leading vehicle.
This introduces an additional factor of model-mismatch as the yielding behavior is human-controlled rather than being driven by the MR-IDM model as the framework (and our previous work \cite{knaup2024active}) assumed, and demonstrates that the proposed approach is robust to such model-mismatches.
\begin{figure}[ht]
    \centering
    \includegraphics[width=0.8\linewidth]{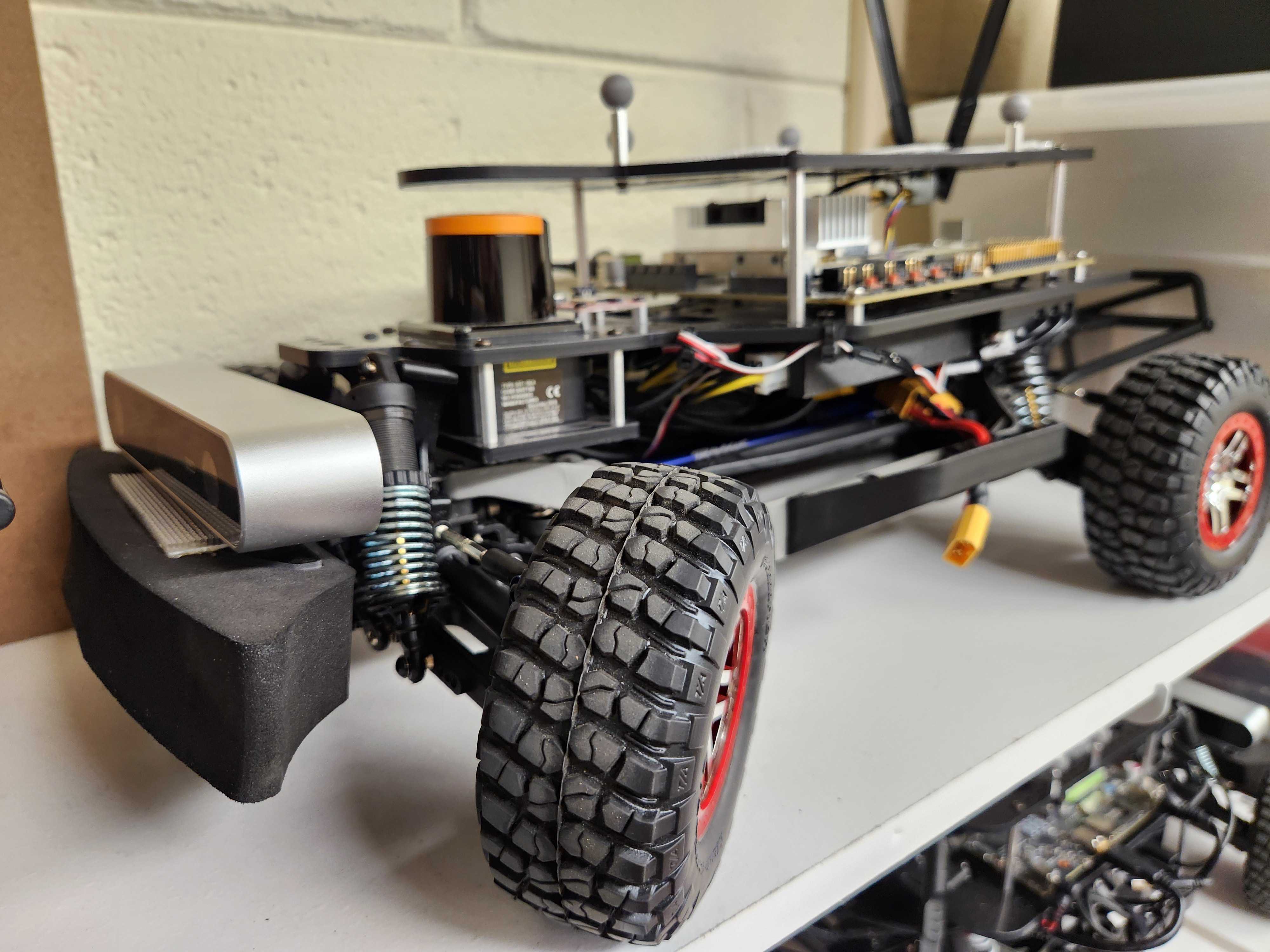}
    \caption{F1-Tenth platform.}
    \label{fig:racecar}
\end{figure}

The proposed approach is compared with two baseline methods: the dual model predictive path integral control (DMPPI) from \cite{knaup2024active} and the ensemble MPPI (EMPPI) from \cite{abraham2020model}.
DMPPI solves the proposed problem~\eqref{prob:saa} while ablating the multi-modal component of the diffusion solver; that is, it utilizes MPPI as a special case of generative diffusion. 
EMPPI additionally ablates the active learning component by utilizing MPPI to solve problem~\eqref{prob:stoch}.
The hyper-parameters of all three algorithms (e.g., number of samples, diffusion steps, etc.) were chosen to maximize performance while allowing for real-time updates at $10$ Hz on the Jetson Orin Nano.
For each method, we conducted $12$ trials featuring different start locations for the ego vehicle relative to the traffic vehicles as well as different variations of traffic vehicle parameters.
Importantly, in each trial, only \emph{one} of the traffic vehicles is ``friendly" and will yield to the ego vehicle. The friendly car assignment is randomized and the ego car does not have prior information about which car is friendly. 
Therefore, the ego vehicle must quickly learn the other vehicles' driving behavior parameters in order to find the friendly driver before reaching the end of the merge zone ($\approx 15$ m). 

Experimental results are shown in Table~\ref{tab:results}, and a video of several trials may be viewed \href{https://youtu.be/Q_JdZuopGL4}{here}.
The proposed approach and Dual MPPI are both able to successfully complete the merge in less than $15$ m in all trials, highlighting the effectiveness of the proposed dual control formulation, which induces ``probing" behaviors.
EMPPI, on the other hand, does not feature active learning and so, without probing behaviors, the vehicle can only merge when a gap opens naturally due to being next to a friendly driver, restricting the vehicle to complete the merge in a purely reactive manner.
As a result, EMPPI only successfully completes the merge $58$\% of the time.
Since EMPPI did not complete the merge in all cases, when calculating the average merge distance, we set the merge distance to $15$ m for the cases in which EMPPI failed to complete the merge in less than $15$ m.

The advantage of the proposed approach over DMPPI is highlighted by the fact that the proposed approach is able to complete the merge much earlier in less than $2$/$3$ the distance as DMPPI, on average, indicating the proposed approach is more efficient and produces better (closer to globally optimal) solutions, which correspond to faster merge completions (and a lower cost). 
Finding the right gap earlier will help in avoiding being too close to the end of the merge ramp in a real world highway on-ramp situation for an AV which can be critical when dealing with short merge ramps and closely spaced traffic. 
Next, we see that the high success rate and faster merging of the proposed approach are achieved without compromising safety, as shown by the fact that all three methods maintain similar minimum distances from the traffic vehicles.
Finally, the higher merge success rates for the active learning frameworks achieve higher success at the expense of higher control effort due to the fact that they need to probe the agents in order to improve their belief of the other agents' behavior online. 

\begin{table}[ht]
\caption{Experimental results over $12$ trials for each method.}
\label{tab:results}

\centering
\begin{tabular}{lccc}
                                           & EMPPI & DMPPI & Proposed    \\ \hline \hline
Merge Success Rate                         & 58\%  & 100\% & 100\% \\ \hline
Ave. Merge Dist. (m)                       & 10.5   & 7.1   & 4.3   \\ \hline
Ave. Min. Distance (m)                     & 0.53  & 0.52  & 0.53 \\ \hline
Ave. Acceleration (m/s$^2$) & 0.04  & 0.05  & 0.07  
\end{tabular}
\end{table}

The efficacy of the proposed approach is further illustrated in Fig.~\ref{fig:snapshots}, which gives an overview of a single trial for all three methods.
In this trial, the ego vehicle starts behind the traffic vehicles and only the green traffic vehicle will yield to the blue ego vehicle. 
Therefore, the ego vehicle must accelerate to align itself longitudinally with the traffic vehicles and determine which car will yield so that the merge may be completed.
As seen in Fig.~\ref{fig:snapshots}, only the dual control methods induce the necessary behavior to position the ego vehicle to probe the traffic vehicles.
Additionally, the proposed algorithm is able to open a gap and complete the merge well ahead of DMPPI by quickly determining that the red vehicle was not ``friendly" and accelerating to probe the next car, whereas DMPPI took much longer to evaluate its belief about the red car.

\begin{figure*}[ht]
    \centering
    \includegraphics[width=0.99\linewidth]{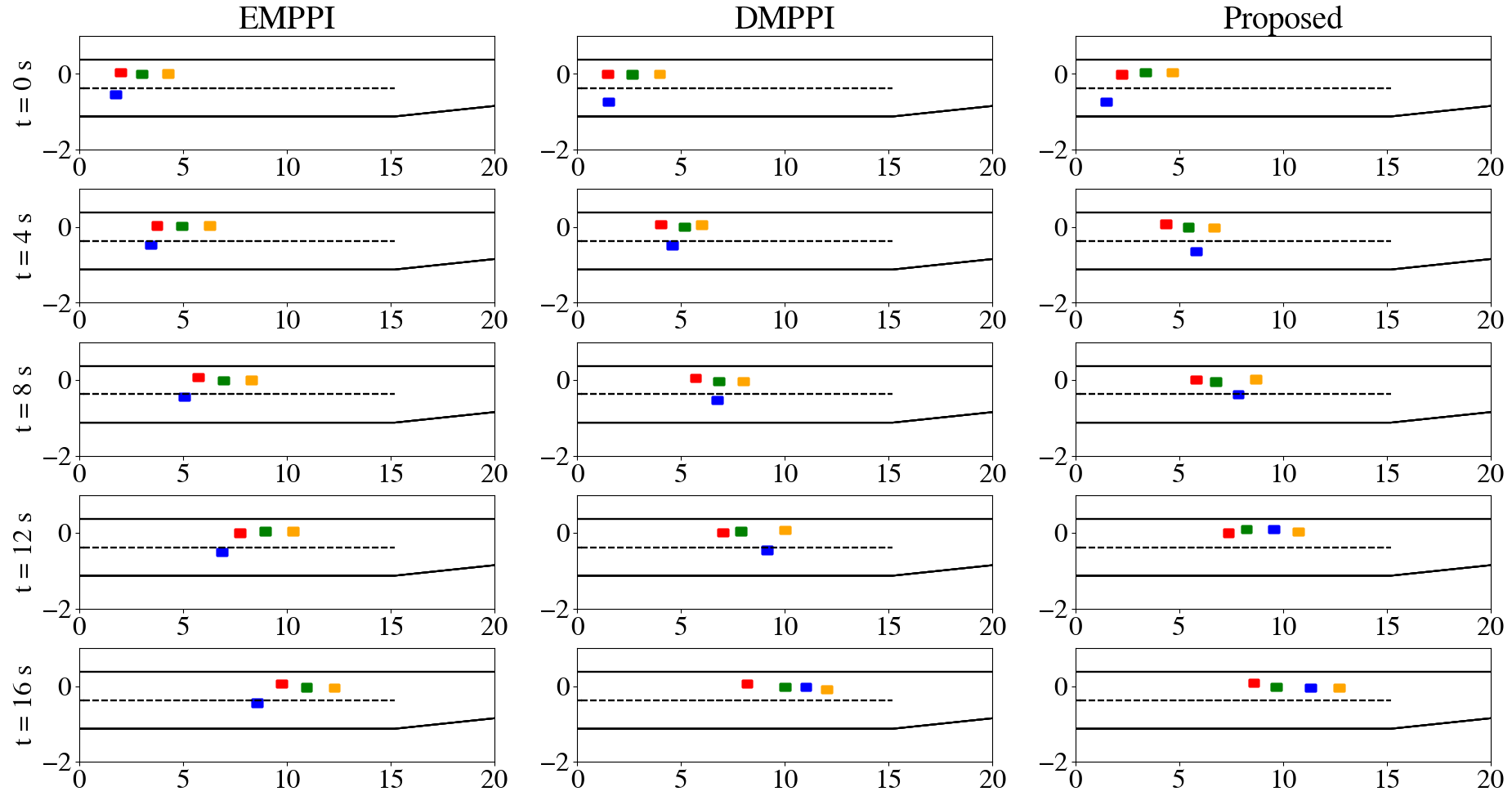}
    \caption{Depiction of single trial for proposed and baseline methods.}
    \label{fig:snapshots}
\end{figure*}

\section{Conclusion} \label{sec:conclusion}

In this paper, we have presented a novel dual control framework for autonomous highway merging, integrating model-based diffusion and online Bayesian inference to enhance interaction-aware planning. 
Our approach actively probes the surrounding human-driven cars to infer their intentions and adapt control actions intelligently. 
A key component of our framework is the novel model-based diffusion solver tailored for receding horizon optimization. 
This solver effectively handles the complexities of non-convex and multimodal interaction problems like merging, contributing to the robustness and adaptability of our approach, while remaining computationally efficient enough to run in real time (10 Hz) on an embedded computer. 
Through several real-world hardware experiments on the F1-Tenth platform, we have demonstrated that our framework improves the efficiency and safety of challenging autonomous driving maneuvers such as merging in dynamic traffic.
This work represents the first application of model-based diffusion to a dual control problem as well as its first application to an autonomous driving problem. 
These contributions advance the field of interaction aware planning and decision-making under uncertainty for autonomous driving, paving the way for more intelligent and adaptive navigation systems. 
Future work will focus on further refining the model-based diffusion solver and exploring its applications in other complex driving scenarios.

\section{Acknowledgments}

The authors thank Rohan Bansal, Shawn Wahi, and Saahir Dhanani for their help with setting up the F1-Tenth hardware experiments. 

\bibliographystyle{IEEEtran}
\bibliography{references}

\end{document}